\documentclass{article}

\usepackage{microtype}
\usepackage{subfigure}
\usepackage{booktabs} 
\usepackage{float}



\usepackage{hyperref}

\usepackage[accepted]{icml2021}

\usepackage{color}
\usepackage{graphicx}
\usepackage{wrapfig}
\usepackage{url}
\usepackage{psfrag}
\usepackage{relsize}
\usepackage{amsmath,amsfonts,mathrsfs,mathtools,amssymb}
\usepackage{threeparttable}
\usepackage{dsfont}

\usepackage{multirow}
\usepackage{array}

\newcommand{\argmax}{\operatornamewithlimits{arg \, max}}

\newtheorem{theorem}{Theorem}

\newtheorem{lemma}{Lemma}

\newtheorem{proof}{Proof}

\usepackage{algorithm}
\usepackage{algorithmic}

\allowdisplaybreaks

\icmltitlerunning{Leveraged Weighted Loss for Partial Label Learning}

\begin{document}
	
	\twocolumn[
	\icmltitle{Leveraged Weighted Loss for Partial Label Learning}
	
	
	
	\icmlsetsymbol{equal}{*}
	
	\begin{icmlauthorlist}
		\textcolor{white}{\icmlauthor{}{pku}}
		\icmlauthor{Hongwei Wen}{ut,equal}
		\icmlauthor{Jingyi Cui}{pku,equal}
		\icmlauthor{Hanyuan Hang}{ut}
		\icmlauthor{Jiabin Liu}{samsung}
		\icmlauthor{Yisen Wang}{pku}
		\icmlauthor{Zhouchen Lin}{pku,pazhou}
	\end{icmlauthorlist}
	
	\icmlaffiliation{ut}{Department of Applied Mathematics, University of Twente, The Netherlands}
	\icmlaffiliation{pku}{Key Lab. of Machine Perception (MoE), School of EECS, Peking University, China}
	\icmlaffiliation{pazhou}{Pazhou Lab, Guangzhou, China}
	\icmlaffiliation{samsung}{Samsung Research China-Beijing, Beijing, China}
	\icmlcorrespondingauthor{Yisen Wang}{yisen.wang@pku.edu.cn}
	\icmlcorrespondingauthor{Jiabin Liu}{Jiabin.liu@samsung.com}
	
	
	\icmlkeywords{Weakly Supervised Learning, Machine Learning, ICML}
	
	\vskip 0.3in
	]
	
	
	
	\printAffiliationsAndNotice{\icmlEqualContribution} 
	
\begin{abstract}
	As an important branch of weakly supervised learning, partial label learning deals with data where each instance is assigned with a set of candidate labels, whereas only one of them is true. 
	Despite many methodology studies on learning from partial labels, there still lacks theoretical understandings of their risk consistent properties under relatively weak assumptions, especially on the link between theoretical results and the empirical choice of parameters. 
	In this paper, we propose a family of loss functions named \textit{Leveraged Weighted} (LW) loss, which for the first time introduces the leverage parameter $\beta$ to consider the trade-off between losses on partial labels and non-partial ones. 
	From the theoretical side, we derive a generalized result of risk consistency for the LW loss in learning from partial labels, based on which we provide guidance to the choice of the leverage parameter $\beta$.
	In experiments, we verify the theoretical guidance, and show the high effectiveness of our proposed LW loss on both benchmark and real datasets compared with other state-of-the-art partial label learning algorithms.
\end{abstract}

\section{Introduction}

Partial label learning \citep{cour2011learning}, also called ambiguously label learning \citep{chen2017learning} and superset label problem \citep{gong2017regularization}, refers to the task where each training example is associated with a set of candidate labels, while only one is assumed to be true. 
It naturally arises in a number of real-world scenarios such as web mining \citep{luo2010learning}, multimedia contents analysis \citep{cour2009learning,zeng2013learning}, ecoinformatics \citep{liu2012conditional}, etc, and subsequently attracts a lot of attention on methodology studies \citep{feng2020provably,zhang2020semi,yao2020network,lyu2019gm,wang2019adaptive}.

As the main target of partial learning lies in disambiguating the candidate labels, two general strategies have been proposed with different assumptions to the latent label space: \textit{1)} Average-based strategy that treats each candidate label equally in the model training phase \citep{hullermeier2006learning, cour2011learning, zhang2015solving}. \textit{2)} Identification-based strategy that considers the ground-truth label as a latent variable, and assume certain parametric model to describe the scores of each candidate label \citep{feng2019partial,yan2020partial,yao2020network}. The former is intuitive but has an obvious drawback that the predictions can be severely distracted by the false positive labels. The latter one attracted lots of attentions in the past decades but is criticized for the vulnerability when encountering differentiated label in candidate label sets. Furthermore, in recent years, more and more literature focuses on making amendments and adjustments on the optimization terms and loss functions on the basis of identification-based model \citep{lv2020progressive, cabannes2020structured, wu2018towards, lyu2019gm, feng2020provably}. 

Despite extensive studies on partial label learning algorithms, theoretically guaranteed ones remain to be the minority. 
Some researchers have studied the statistical consistency \citep{cour2011learning, feng2020provably, cabannes2020structured} and the learnability \citep{liu2014learnability} of partial label learning algorithms. 
However, these theoretical studies are often based on rather strict assumptions, e.g. convexity of loss function \citep{cour2011learning}, uniformly sampled partial label sets \citep{feng2020provably}, etc.
Moreover, it remains to be an open problem why an algorithm performs better than others under specific parameter settings, or in other words, how can theoretical results guide parameter selections in computational implementations.

In this paper, we aim at investigating further theoretical explanations for partial label learning algorithms. 
Applying the basic structure of identification-based methods, we propose a family of loss functions named \textit{Leveraged Weighted} (LW) loss. 
From the perspective of risk consistency, we provide theoretical guidance to the choice of the leverage parameter in our proposed LW loss by discussing the supervised loss to which LW is risk consistent. 
Then we design the partial label learning algorithm by iteratively identifying the weighting parameters.
As follows are our contributions:
\begin{itemize}
\item We propose a family of loss function for partial label learning, named the Leveraged Weighted (LW) loss function, where we for the first time introduce the leverage parameter $\beta$ that considers the trade-offs between losses on partial labels and non-partial labels.

\item We for the first time generalize the uniform assumption on the generation procedure of partial label sets, 
under which we prove the risk consistency of the LW loss. 
{We also prove the Bayes consistency of our LW loss.}
Through discussions on the supervised loss to which LW is risk consistent, we obtain the potentially effective values of $\beta$.

\item We present empirical understandings to verify the theoretical guidance to the choice of $\beta$,
and experimentally demonstrate the effectiveness of our proposed algorithm based on the LW loss over other state-of-the-art partial label learning methods on both benchmark and real datasets. 
\end{itemize}

\section{Related Works}

We briefly review the literature for partial label learning.

\textbf{Average-based methods.}
The average-based methods normally consider each candidate label as equally important during model training, and average the outputs of all the candidate labels for predictions. 
Some researchers apply nearest neighbor estimators and predict a new instance by voting \citep{hullermeier2006learning,zhang2015solving}. Others further take advantage of the information in non-candidate samples. For example, \cite{cour2011learning,zhang2016partial} employ parametric models to demonstrate the functional relationship between features and the ground truth label. The parameters are trained to maximize the average scores of candidate labels minus the average scores of non-candidate labels.

\textbf{Identification-based methods.}
The identification-based methods aim at directly maximizing the output of exactly one candidate label, chosen as the truth label. A wealth of literature adopt major machine learning techniques such as maximum likelihood criterion \cite{Jin2002multiple, liu2012conditional} and maximum margin criterion \cite{nguyen2008classification, yu2016maximum}.
As deep neural networks (DNNs) become popular, DNN-based methods outburst recently. \cite{feng2019partial} introduces self-learning with network structure; \cite{yan2020partial} studies the utilization of batch label correction; \cite{yao2020network} manages to improve the performance by combining different networks. Moreover, it is worth highlighting that these algorithms have shown their weaknesses when facing the false positive labels that co-occur with the ground truth label.

\textbf{Binary loss-based multi-class classification.}
	Building multi-class classification loss from multiple binary ones is a general and frequently used scheme. 
	In previous works, to extend margin-based binary classifiers (e.g., SVM and AdaBoost) to the multi-class setting, they adopted the combination of binary classification losses using constraint comparison \citep{lee2004multicategory, zhang2004statistical}, loss-based decoding \citep{allwein2000reducing}, etc.
	In this paper, inspired by these losses for multi-class classification, we design a loss function for multi-class partial label learning via multiple binary loss functions.

In this paper, we follow the idea of the identification-based method, propose the LW loss function, and provide theoretical results on risk consistency.
This result gives theoretical insights into the problem why an algorithm shows better performance under certain parameter settings than others.

\section{Methodology}\label{sec::methodology}

In this section, we first introduce some background knowledge about learning with partial labels in Section \ref{sec::prelims}. 
Then in Section \ref{sec::LWloss} we propose a family of LW loss function for partial labels.
In Section \ref{sec::thm_interpret}, we prove the risk consistency of the LW loss and present guidance to the empirical choice of the leverage parameter $\beta$.
Finally, we present our proposed practical algorithm in Section \ref{sec::main_alg}.

\subsection{Preliminaries}\label{sec::prelims}

\textbf{Notations.}
Denote $\mathcal{X} \subset \mathbb{R}^d$ as a non-empty feature space (input space), $\mathcal{Y} = [K] := \{1, \ldots, K\}$ as the \textit{supervised label space}, where $k$ is the number of classes, and $\vec{\mathcal{Y}} := \{\vec{y} \, |\, \vec{y} \subset \mathcal{Y} \} = 2^{[K]}$ as the \textit{partial label space}, where $2^{[K]}$ is the collection of all subsets in $[K]$.
For the rest of this paper, $y$ denotes the true label of  $x$ unless otherwise specified.

\textbf{Basic settings.}
In learning with partial labels, an input variable $X \in \mathcal{X}$ is associated with a set of potential labels $\vec{Y} \in \vec{\mathcal{Y}}$ instead of a unique true label $Y \in \mathcal{Y}$.
The goal is to find the latent \textit{ground-truth label} $Y$ for the input $X$ through observing the \textit{partial label set} $\vec{Y}$. 
The basic definition for partially supervised learning lies in the fact that the true label $Y$ of an instance $X$ must always reside in the partial label set $\vec{Y}$, i.e.
\begin{align}
	\mathrm{P}(y \in \vec{Y} \,|\, Y=y, x) = 1.
\end{align}
That is, we have $\#|\vec{Y}| \geq 1$, and $\#|\vec{Y}| = 1$ holds if and only if $\vec{Y} = \{y\}$, in which case the partial label learning problem reduces to multi-class classification with supervised labels.


\textbf{Risk consistency.}
\textit{Risk consistency} is an important tool in studying weakly supervised algorithms \citep{ishida2017learning, ishida2019complementary, feng2020learning, feng2020provably}. 
We say a method is \textit{risk-consistent} if its corresponding classification \textit{risk}, also called \textit{generalization error}, is equivalent to the supervised classification risk $\mathcal{R}(f)$ given the same classifier $f$. 
Note that risk consistency implies \textit{classifier consistency} \citep{xia2019anchor}, where learning from partial labels results in the same optimal classifier as that when learning from the fully supervised data.

To be specific, denote $g(x) = (g_1(x), \ldots, g_K(x))$ as the score function learned by an algorithm, where $g_z(x)$ is the score function for label $z \in [K]$. Larger $g_z(x)$ implies that $x$ is more likely to come from class $z \in [K]$. 
Then the resulting classifier is $f(x) = \argmax_{z \in [K]} g_z(x)$.
By definition, we denote 
\begin{align}\label{eq::supervisedrisk}
	\mathcal{R}(\mathcal{L}, g) := \mathbb{E}_{(X,Y)}[\mathcal{L}(Y,g(X))],
\end{align}
as the \textit{supervised risk} w.r.t. \textit{supervised loss function} $\mathcal{L} : \mathcal{Y} \times \mathbb{R}^K \to \mathbb{R}^+$ for supervised classification learning. 
On the other hand, we denote
\begin{align}
\bar{\mathcal{R}}(\bar{\mathcal{L}}, g) 
:= \mathbb{E}_{(X,\vec{Y})}[\bar{\mathcal{L}}(\vec{Y},g)]
\end{align}
as the \textit{partial risk} w.r.t. \textit{partial loss function} $\bar{\mathcal{L}} : \vec{\mathcal{Y}} \times \mathbb{R}^K \to \mathbb{R}^+$,
measuring the expected loss of $g$ learned through partial labels w.r.t. the joint distribution of $(X, \vec{Y})$.
Then a partial loss $\bar{\mathcal{L}}$ is \textit{risk-consistent} to the supervised loss $\mathcal{L}$ if $\bar{\mathcal{R}}(\bar{\mathcal{L}}, g) = \mathcal{R}(\mathcal{L}, g)$.

{
\textbf{Bayes consistency.}
We denote $g_L^*:=\sup_{g \in \mathcal{M}} \mathcal{R}(\mathcal{L}, g)$ as the Bayes decision function w.r.t. the loss function $\mathcal{L}$, where $\mathcal{M}$ contains all measurable functions and $\mathcal{R}_{\mathcal{L}}^* :=\mathcal{R}(\mathcal{L}, g^*)$. 
Similarly, we denote $\mathcal{R}^* := \mathcal{R}_{\mathcal{L}_{0\text{-}1}}^*$ as the Bayes decision function w.r.t. the multi-class $0$-$1$ loss, i.e. 
\begin{align}\label{eq::01loss}
\mathcal{L}_{0\text{-}1}(y, g(x)) := \mathbf{1}\{\argmax_{k\in[K]} g_k(x) \neq y\},
\end{align}
where $\mathbf{1} \{\cdot\}$ denotes the indicator function.
Then if there exist a collection $\{g_n\}$ such that $\mathcal{R}(\mathcal{L}_{0\text{-}1}, g_n) \to \mathcal{R}^*$ as $n\to\infty$, we say that the surrogate loss $\mathcal{L}$ reaches \textit{Bayes risk consistency}.
}

\subsection{Leveraged Weighted (LW) Loss Function}\label{sec::LWloss}

In this paper, we propose a family of loss function for partial label learning named \textit{Leveraged Weighted} (LW) loss function.
We adopt a multiclass scheme frequently used for the fully supervised setting \citep{crammer2001algorithmic, rifkin2004defense, zhang2004statistical, tewari2005consistency}, 
which combines binary losses $\psi(\cdot): \mathbb{R} \to \mathbb{R}^+$, a non-increasing function, to create a multiclass loss.
We highlight that it is the first time that the leverage parameter $\beta$ is introduced into loss functions for partial label learning, which leverages between losses on partial labels and non-partial ones.
To be specific, the partial loss function of concern is of the form 
\begin{align}\label{eq::partial_loss}
	\bar{\mathcal{L}}_{\psi} (\vec{y},g(x)) = \sum_{z\in\vec{y}} w_z \psi(g_z(x)) + \beta \cdot \sum_{z\notin\vec{y}} w_z \psi(-g_z(x)),
\end{align} 
where $\vec{y} \in \vec{\mathcal{Y}}$ denotes the partial label set.
It consists of three components.
\begin{itemize}
	\item A binary loss function $\psi(\cdot) : \mathbb{R} \to \mathbb{R}^+$, where $\psi(g_z(x))$ forces $g_z$ to be larger when $z$ resides in the partial label set $\vec{y}$, while $\psi(-g_z(x))$ punishes large $g_z$ when $z \notin \vec{y}$.
	\item Weighting parameters $w_z \geq 0$ on $\psi(g_z)$ for $z \in [K]$. Generally speaking, we would like to assign more weights to the loss of labels that are more likely to be the true label.
	\item The leverage parameter $\beta \geq 0$ that distinguishes between partial labels and non-partial ones. Larger $\beta$ quickly rules out non-partial labels during training, while it also lessens weights assigned to partial labels. 
\end{itemize}

We mention that the partial loss proposed in \eqref{eq::partial_loss} is a general form.
Some special cases include

	\textit{1)} Taking $\beta = 0$, $w_z = 1/\#|\vec{y}|$ for $z \in \vec{y}$, we achieve the partial loss proposed by \cite{Jin2002multiple}, the form of which is
	\begin{align}
		\frac{1}{\#|\vec{y}|}\sum_{y\in\vec{y}}\psi(g_y(x)).
	\end{align}

	\textit{2)} Taking $\beta = 0$, and $w_{z^*} = 1$ where $z^* = \argmax_{z \in \vec{y}} g_z$, $w_z = 0$ for $z \in \vec{y} \setminus \{z^*\}$, we achieve the partial loss function proposed by 
	\cite{lv2020progressive}, 
	with the form
	\begin{align}\label{eq::softmax}
		\psi(\max_{y\in\vec{y}}g_y(x)) \Leftrightarrow \min_{y\in\vec{y}} \psi(g_y(x)).
	\end{align}

	\textit{3)} By taking $\beta = 1$, and $w_{z^*} = 1$ where $z^* = \argmax_{z \in \vec{y}} g_z$, $w_z = 0$ for $z \in \vec{y} \setminus \{z^*\}$, $w_z = 1$ for $z \notin \vec{y}$, 
	we achieve the partial loss function proposed by \cite{cour2011learning}, with the form
	\begin{align}
		\psi(\max_{y\in\vec{y}}g_y(x))+\sum_{y\notin\vec{y}}\psi(-g_y(x)).
	\end{align}

\subsection{Theoretical Interpretations} \label{sec::thm_interpret}

In this part, we first relax the assumption on the generation procedure of the partial label set, and show the risk consistency of our proposed LW loss function. 
Then by observing the supervised loss to which LW is risk consistent, we study the leverage parameter $\beta$ and deduce its reasonable values.
All proofs are shown in Section \ref{sec::Aproofs} of the supplements.

\subsubsection{Generalizing the Uniform Sampling Assumption}
In previous study of risk consistency, the partial label set $\vec{Y}$ is assumed to be independently and uniformly sampled given a specific true label $Y$ \citep{feng2020provably}, i.e. 
\begin{align}\label{eq::uniform_dist}
	\mathrm{P}(\vec{Y}=\vec{y} \,|\, Y=y, x) = \left\{
	\begin{aligned}
		&\frac{1}{2^{k-1}-1}, &\text{ if } y \in \vec{y},\\
		&0, &\text{ otherwise. }
	\end{aligned}
\right.
\end{align}
Note that this data generation procedure is equivalent to assuming
$\mathrm{P}(y \in \vec{Z} \,|\, x) = \frac{1}{2}$,
where $\vec{Z}$ is an unknown label set uniformly sampled from $[K]$.
The intuition behind is that if no information of $\vec{Z}$ is given, we may randomly guess with even probabilities whether the correct $y$ is included in an unknown label set $\vec{Z}$ or not.

However, in real-world situations, some combination of partial labels may be more likely to appear than others. 
Instances belonging to certain classes usually share similar features e.g. images of \textit{dog} and \textit{cat} may look alike, while they may be less similar to images of \textit{truck}. 
Thus, given these shared features indicating the true label of an instance, the probability of label $z \neq y$ entering the partial label set may be different. 
For instance, when the true label is \textit{dog}, \textit{cat} is more likely to be picked as a partial label than \textit{truck}.

Therefore, in this paper, we generalize the uniform sampling of partial label sets, and allow the sampling probability to be \textit{label-specific}.
Denote $q_z \in [0,1]$ as 
\begin{align}
	q_z := \mathrm{P}(z \in \vec{Y} \,|\, Y = y, x),
\end{align}
for $z \in [K]$.
Then for $z =y$, we have $q_y = 1$ according to the problem settings of learning from partial labels, and for $z \neq y$, we have $q_z < 1$ due to the small ambiguity degree condition \citep{cour2011learning}, which guarantees the ERM learnability of partial label learning problems \citep{liu2014learnability, lv2020progressive}.
Then when the elements in $\vec{y}$ is assumed to be independently drawn, the conditional distribution of the partial label set $\vec{Y}$ turns out to be
	\begin{align}
		\mathrm{P}(\vec{Y}=\vec{y} \,|\, Y=y, x) 
		= \prod_{s \in \vec{y}, s \neq y} q_s \cdot \prod_{t \notin \vec{y}} (1-q_t).
	\end{align}
		where $y$ is the true label of input $x$.

Note that the above generation procedure of the partial label set allows the existence of $[K]$ to be a partial label set. 
If we want to rule out this set, we can simply drop it and sample the partial label set again.
By this means, the conditional distribution becomes
\begin{align*}
	\mathrm{P}(\vec{Y}=\vec{y} \,|\, Y=y, x) 
	= \frac{1}{1-M} \prod_{s \in \vec{y}, s \neq y} q_s \cdot \prod_{t \notin \vec{y}} (1-q_t),
\end{align*}
where $M = \prod_{z \neq y} q_z$.
Taking the special case where $q_z = 1/2$ for all $z \neq y$, we reduce to the generation procedure \eqref{eq::uniform_dist} as in \cite{feng2020provably}.

%

\subsubsection{Risk-consistent Loss Function}

Under the above generation procedure, we take a deeper look at our proposed LW loss and prove its risk consistency.

\begin{theorem}\label{thm::weight}
The LW partial loss function proposed in \eqref{eq::partial_loss} is risk-consistent with respect to the supervised loss function with the form
\begin{align}\label{eq::lv2020progressive_loss}
\mathcal{L}_{\psi}(y,g&(x)) = w_y \psi(g_y(x)) 
\nonumber\\
&+ \sum_{z \neq y} w_z q_z \big[\psi(g_z(x)) + \beta \psi(-g_z(x))\big].
\end{align}
\end{theorem}

Theorem \ref{thm::weight} indicates the existence of a loss function $\mathcal{L}_{\psi}$ for supervised learning to which the LW loss $\bar{\mathcal{L}}_{\psi}$ is risk consistent.
Note that the resulting form of the supervised loss function \eqref{eq::lv2020progressive_loss} is a widely used multi-class scheme in supervised learning, e.g. \citet{crammer2001algorithmic, rifkin2004defense, tewari2007consistency}.

It is worth mentioning that this is the first time that a risk consistency analysis is conducted under a label-specific sampling of the partial label set.
Moreover, compared with \citet{lv2020progressive}, where the proposed loss function is proved to be classifier consistent under the deterministic scenario, our result on risk consistency is a stronger claim and applies to both deterministic and stochastic scenarios.

{
The next theorem shows that as long as $\beta>0$, the supervised risk induced by \eqref{eq::lv2020progressive_loss} is consistent to the Bayes risk $\mathcal{R}^*$.
That is, optimizing the supervised loss in \eqref{eq::lv2020progressive_loss} can result in the Bayes classifier under $0\text{-}1$ loss.
\begin{theorem}\label{thm::calibration}
	 Let $\mathcal{L}_{\psi}$ be of the form in \eqref{eq::lv2020progressive_loss} and $\mathcal{L}_{0\text{-}1}$ be the multi-class $0$-$1$ loss. 
	 Assume that $\psi(\cdot)$ is differentiable and symmetric, i.e. $\psi(g_z(x)) + \psi(-g_z(x)) = 1$. 
	 For $\beta>0$, if there exist a sequence of functions $\{\hat{g}_n\}$ such that 
	\begin{align*}
		\mathcal{R}(\mathcal{L}_{\psi},\hat{g}_n) \to \mathcal{R}^*_{\mathcal{L}_{\psi}},
	\end{align*}
	then we have 
	 \begin{align*}
	 \mathcal{R}(\mathcal{L}_{0\text{-}1},\hat{g}_n) \to \mathcal{R}^*.
	 \end{align*}
\end{theorem}
Combined with Theorem \ref{thm::weight}, when $\beta > 0$, we have our LW loss consistent to the Bayes classifier.
}

\subsubsection{Guidance on the Choice of $\beta$}

In this section, we try to answer the question why we should choose some certain values of $\beta$ for the LW loss $\bar{\mathcal{L}}_{\psi}$ instead of others when learning from partial labels. 
Recall that when minimizing a risk consistent partial loss function in partial label learning, we are at the same time minimizing the corresponding supervised loss.
Therefore, by Theorem \ref{thm::weight}, a satisfactory supervised loss $\mathcal{L}_{\psi}$ in supervised learning naturally corresponds to an LW loss $\bar{\mathcal{L}}_{\psi}$ with the desired value of the leverage parameter $\beta$ in partial label learning.

When we take a closer look at the right-hand side of \eqref{eq::lv2020progressive_loss}, the loss function $\mathcal{L}_{\psi}$ to which LW loss is risk-consistent always contains the term $\psi(g_y)$, which focuses on identifying the true label $y$.
On the other hand, an interesting finding is that the leverage parameter $\beta$ determines the relative scale of $\psi(g_z)$ and $\psi(-g_z)$ for all $z \neq y$, while it does not affect the loss on the true label $y$.
 

In the following discussions, we focus on symmetric binary loss $\psi(\cdot)$, where $\psi(g_z(x)) + \psi(-g_z(x)) = 1$, for their fine theoretical properties.
We remark that commonly adopted loss functions such as zero-one loss, Sigmoid loss, Ramp loss, etc. satisfy the symmetric condition.
In what follows, we present the results of risk consistency for LW loss with specific values of $\beta$, and discuss each case respectively. 


\textbf{Case 1:}
When $\beta = 0$ \citep[e.g.][]{lv2020progressive},
the LW loss function $\bar{\mathcal{L}}_{\psi}$ is risk-consistent to
\begin{align}
	w_y \psi(g_y(x)) + \sum_{z \neq y} w_z q_z \psi(g_z(x)).
\end{align}
In this case, in addition to focusing on the true label $y$, $\mathcal{L}_{\psi}$ also gives positive weights to the untrue labels as long as there exists a label $z \neq y$ such that $w_z > 0$.
Since the minimization of $\psi(g_z)$ may lead to false identification of label $z \neq y$, $\beta = 0$ is not preferred for LW loss.


\textbf{Case 2:}
When $\beta = 1$ \citep[e.g.][]{Jin2002multiple, cour2011learning}, 
the LW loss function $\bar{\mathcal{L}}_{\psi}$ is risk-consistent to
\begin{align}
	w_y \psi(g_y(x)) + \sum_{z \neq y} w_z q_z.
\end{align}
In this case, the minimization of $\bar{\mathcal{L}}_{\psi}$ indicates the minimization of $\mathcal{L}_{\psi} = \psi(g_y(x))$, aiming at directly identifying the true label $y$. 
The idea is similar to that of the cross entropy loss, where $\mathcal{L}_{CE}(y,g(x)) := -\log(g_y(x))$ .
Therefore, we take $\beta = 1$ as a reasonable choice for LW loss.

\textbf{Case 3:}
When $\beta = 2$, 
the LW loss function $\bar{\mathcal{L}}_{\psi}$ is risk-consistent with 
\begin{align}
	w_y \psi(g_y(x)) &+ \sum_{z \neq y} w_z q_z \psi(-g_z(x)) 
	+ \sum_{z \neq y} w_z q_z.
\end{align}
In this case, the LW loss not only encourages the learner to identify the true label $y$ by minimizing $\psi(g_y)$, but also helps rule out the untrue labels $z \neq y$ by punishing large value of $\psi(-g_z)$.
Moreover, for a confusing label $z \neq y$ that is more likely to appear in the partial label set, i.e. $q_z$ is larger, $\mathcal{L}_{\psi}$ imposes severer punishment on $g_z$. 
Therefore, $\beta = 2$ is also a preferred choice for LW loss.
Especially, when taking $w_z = 1/q_z$ for $z \in [K]$, we achieve the form
\begin{align}
\psi(g_y(x)) + \sum_{z\neq y}\psi(-g_z(x)) + K-1,
\end{align}
which exactly corresponds to the \textit{one-versus-all} (OVA) loss function proposed by 
\citet{zhang2004statistical}.

{
To conclude, it is not a good choice for LW loss to take $\beta = 0$, as most commonly used loss functions do.
Our theoretical interpretations of risk consistency show that $\beta>0$ and especially $\beta>1$ are preferred choices, which are also empirically verified in Section \ref{sec::parameter}.
}

\subsection{Practical Algorithm}\label{sec::main_alg}


In the theoretical analysis in the previous section, we focus on partial and supervised loss functions that are consistent in risk.
However, in experiments, the risk for partial label loss is not directly accessible since the underlying distribution of $\mathrm{P}(X, \vec{Y})$ is unknown. 
Instead, on the partially labeled sample 
$D_n:=\{(x_1,\vec{y}_1), \ldots, (x_n,\vec{y}_n)\}$, we try to minimize the empirical risk 
of a learning algorithm defined by
\begin{align}\label{eq::empirical_risk}
	\bar{\mathcal{R}}_{D_n}(\bar{\mathcal{L}},g(X))
	= \frac{1}{n} \sum_{i=1}^n \bar{\mathcal{L}}(\vec{y}_i, g(x_i)).
\end{align}
Moreover, in this part we take the network parameters $\theta$ for score functions $g(x):= \big(g_1(x),\ldots, g_K(x)\big)$ into consideration, and write $g(x;\theta)$ and $g_z(x;\theta)$ instead.

\textbf{Determination of weighting parameters.}
Since our goal is to find out the unique true label after observing partially labeled data, we'd like to focus more on the true label contained in the partial label set, while ruling out the most confusing one outside this set. 
Therefore, we assign larger weights to $\psi(g_y(x))$, where $y$ denotes the true label of $x$, and to $\psi(-g_z(x))$, where $z$ is the non-partial label with the highest score among $[K] \setminus \vec{y}$.

However, since we cannot directly observe the true label $y$ for input $x$ from the partially labeled data, the weighting parameters cannot be directly assigned. 
Therefore, inspired by the EM algorithm \citep{dempster1977maximum} and PRODEN \citep{lv2020progressive}, we learn the weighting parameters through an iterative process instead of assigning fixed values. 

To be specific, at the $t$-th step, given the network parameters $\theta^{(t)}$, we calculate the weighting parameters by respectively normalizing the score functions $g_z(x;\theta)$ for $z \in \vec{y}$ and those for $z \notin \vec{y}$, i.e.
\begin{align}
w_z^{(t)} &= \frac{\exp(g_z(x;\theta^{(t)}))}{\sum_{z \in \vec{y}} \exp(g_z(x;\theta^{(t)}))} \text{ for }z \in \vec{y}, \text{ and }
\label{eq::weights1}
\\
w_z^{(t)} &= \frac{\exp(g_z(x;\theta^{(t)}))}{\sum_{z \notin \vec{y}} \exp(g_z(x;\theta^{(t)}))} \text{ for }z \notin \vec{y}.
\label{eq::weights2}
\end{align} 
By this means we have $\sum_{z\in\vec{y}} w_z^{(t)} = \sum_{z\notin\vec{y}} w_z^{(t)}= 1$.
Note that $w_z^{(t)}$ varies with sample instances. Thus for each instance $(x_i, \vec{y}_i)$, $i=1, \ldots, n$, we denote the weighting parameter as $w_{z,i}^{(t)}$.
As a special reminder, we initialize 
	$w_{z,i}^{(0)} = \frac{1}{\#|\vec{y}_i|}$ for $z \in \vec{y}_i$ and
	$w_{z,i}^{(0)} = \frac{1}{K-\#|\vec{y}_i|}$ for $z \notin \vec{y}_i$.

The intuition behind the respective normalization is twofold. 
First of all, by respectively normalizing scores of partial labels and non-partial ones, we achieve our primary goal of focusing on the true label and the most confusing non-partial label.
Secondly, if we simply perform normalization on all score functions, the weights for partial labels tend to grow rapidly through training, resulting in much larger weights for the partial losses than the non-partial ones. 
Thus, as the training epochs grow, the losses on non-partial labels as well as the leverage parameter $\beta$ gradually become ineffective, which we are not pleased to see.

The main algorithm is shown in Algorithm \ref{alg::LW}.
{
Note that here $\beta$ is a hyper-parameter tuned by validation while $w$ is the parameter trained through data.
}

\begin{algorithm}[htbp!]
	\caption{LW Loss for Partial Label Learning}
	\label{alg::LW}
	\begin{algorithmic}
		\STATE {\bfseries Input:} 
		Training data $D_n := \{(x_1, \vec{y}_1),\ldots,(x_n,\vec{y}_n)\}$;\\
		\quad\quad\quad Leverage parameter $\beta$; \\
		\quad\quad\quad Learning rate $\rho$;\\
		\quad\quad\quad Number of Training Epochs $T$;\\
		\FOR{$t =1 \textbf{ to } T$}
		\STATE Calculate $\bar{\mathcal{R}}_{D_n}^{(t)} (\bar{\mathcal{L}}^{(t-1)},g(x;\theta^{(t-1)}))$ by \eqref{eq::empirical_risk};
		\STATE Update network parameter $\theta^{(t)}$ and achieve $g(x;\theta^{(t)})$.
		\STATE Update weighting parameters $w_{z,i}^{(t)}$ by \eqref{eq::weights1} and \eqref{eq::weights2};
		\ENDFOR
		\STATE {\bfseries Output:}
		Decision function
			$\argmax_{z \in [K]} g_z(x;\theta^{(T)})$.
	\end{algorithmic}
\end{algorithm}

\section{Experiments}

In this part, we empirically verify the effectiveness of our proposed algorithm through performance comparisons as well as other empirical understandings.

\subsection{The Classification Performance}\label{sec::comparisons}

In this section, we conduct empirical comparisons with other state-of-the-art partial label learning algorithms on both benchmark and real datasets.

\begin{table*}[!h] 
	\setlength{\tabcolsep}{9pt}
	\centering
	\caption{Accuracy comparisons on benchmark datasets.}
	\label{tab::SynCompare}
	\resizebox{\textwidth}{!}{
		\begin{tabular}{cl clll}
			\toprule
			Dataset & Method & Base Model& \quad\ \ $q=0.1$ & \quad\ \ $q=0.3$ & \quad\ \ $q=0.5$ \\
			\midrule
			\multirow{5}*{MNIST} 
			& RC & MLP & $98.44 \pm 0.11\%*$ & $98.29 \pm 0.05\%*$ & $98.14 \pm 0.03\%*$ \\
			& CC & MLP & $98.56 \pm 0.06\%*$ & $98.32 \pm 0.06\%*$ & $98.21 \pm 0.07\%*$ \\
			& PRODEN & MLP & $98.57 \pm 0.07\%*$ & $98.48 \pm 0.10\%*$ & $98.40 \pm 0.15\%*$ \\
			& LW-Sigmoid & MLP & $\underline{98.82 \pm 0.04\%}$ & $\underline{98.74 \pm 0.07\%}$ & $\underline{98.55 \pm 0.07\%}$ \\
			& LW-Cross entropy & MLP & $\mathbf{98.89 \pm 0.06\%}$ & $\mathbf{98.81 \pm 0.06\%}$ & $\mathbf{98.59 \pm 0.15\%}$ \\
			\hline
			\multirow{5}*{Fashion-MNIST} 
			& RC & MLP & $89.69 \pm 0.08\%*$ & $89.47 \pm 0.04\%*$ & $\underline{88.97 \pm 0.06\%}*$ \\
			& CC & MLP & $89.63 \pm 0.10\%*$ & $89.11 \pm 0.19\%*$ & $88.31 \pm 0.14\%*$ \\
			& PRODEN & MLP & $89.62 \pm 0.13\%*$ & $89.17 \pm 0.08\%*$ & $88.72 \pm 0.18\%*$ \\
			& LW-Sigmoid & MLP & $\underline{90.25 \pm 0.16\%}$ & $\underline{89.67 \pm 0.15\%}*$ & $88.76 \pm 0.03\%*$ \\
			& LW-Cross entropy & MLP & $\mathbf{90.52 \pm 0.19\%}$ & $\mathbf{90.15 \pm 0.13\%}$ & $\mathbf{89.54 \pm 0.10\%}$ \\
			\hline
			\multirow{5}*{Kuzushiji-MNIST} 
			& RC & MLP & $92.12 \pm 0.17\%*$ & $91.83 \pm 0.18\%*$ & $90.84 \pm 0.26\%*$ \\
			& CC & MLP & $92.57 \pm 0.14\%*$ & $92.08 \pm 0.06\%*$ & $90.58 \pm 0.18\%*$ \\
			& PRODEN & MLP & $92.20 \pm 0.43\%*$ & $91.18 \pm 0.15\%*$ & $89.64 \pm 0.32\%*$ \\
			& LW-Sigmoid & MLP & $\underline{93.63 \pm 0.39\%}$ & $\underline{92.92 \pm 0.28\%}*$ & $\underline{91.81 \pm 0.25\%}*$ \\
			& LW-Cross entropy & MLP & $\mathbf{94.14 \pm 0.12\%}$ & $\mathbf{93.57 \pm 0.13\%}$ & $\mathbf{92.30 \pm 0.23\%}$ \\
			\hline
			\multirow{5}*{CIFAR-10} & RC & ConvNet & $86.53 \pm 0.12\%*$ & $85.90 \pm 0.13\%*$ & $84.48 \pm 0.17\%*$ \\
			& CC & ConvNet & $86.47 \pm 0.22\%*$ & $85.33 \pm 0.19\%*$ & $82.74 \pm 0.22\%*$ \\
			& PRODEN & ConvNet & $89.71 \pm 0.13\%*$ & $88.57 \pm 0.10\%*$ & $85.95 \pm 0.14\%*$ \\
			& LW-Sigmoid & ConvNet & $\mathbf{90.88 \pm 0.09\%}$ & $\mathbf{89.75 \pm 0.08\%}$ & $\underline{87.27 \pm 0.15\%}*$ \\
			& LW-Cross entropy & ConvNet & $\underline{90.58 \pm 0.04\%}*$ & $\underline{89.68 \pm 0.10\%}$ & $\mathbf{88.31 \pm 0.09\%}$ \\
			\bottomrule 
		\end{tabular}
	}
	\begin{tablenotes}
		\footnotesize
		\item  The best results are marked in \textbf{bold} and the second best marked in \underline{underline}. The standard deviation is also reported. 
		We use $*$ to represent that the best method is significantly better than the other compared methods.
	\end{tablenotes}
\end{table*}

\begin{table*}[!h] 
	\setlength{\tabcolsep}{9pt}
	\centering
	\caption{Accuracy comparisons on real datasets.}
	\label{tab::RealCompare}
	\resizebox{\textwidth}{!}{
		\begin{tabular}{l lllll}
			\toprule
			\multirow{2}*{Method} 
			& \multicolumn{5}{c}{Dataset} \\
			\cline{2-6}
			& \qquad\ Lost & \quad\  MSRCv2 & \quad \ Birdsong & \ \ SoccerPlayer & \ \ YahooNews \\
			\midrule
			IPAL & $62.37 \pm 4.81\%*$ & $50.34 \pm 3.24\%*$ & $70.20 \pm 4.62\%*$ & $55.79 \pm 0.88\%*$ & $64.57 \pm 1.51\%*$ \\
			PALOC & $57.80 \pm 7.00\%*$ & $47.51 \pm 3.78\%*$ & $70.20 \pm 3.79\%*$ & $53.96 \pm 2.38\%*$ & $60.36 \pm 1.48\%*$ \\
			PLECOC & $63.04 \pm 6.72\%*$ & $44.13 \pm 5.06\%*$ & $73.88 \pm 3.41\%$ & $29.39 \pm 9.38\%*$ & $60.41 \pm 1.50\%*$ \\
			\hline
			RC-Linear & $75.93 \pm 3.62\%*$ & $45.82 \pm 4.74\%*$ & $71.73 \pm 2.84\%*$ & $57.00 \pm 2.15\%$ & $67.42 \pm 1.11\%*$ \\
			CC-Linear & $75.57 \pm 3.58\%*$ & $45.56 \pm 3.97\%*$ & $71.83 \pm 2.85\%*$ & $56.75 \pm 1.87\%*$ & $67.43 \pm 1.07\%*$ \\
			PRODEN-Linear & $76.33 \pm 4.51\%$ & $44.04 \pm 4.50\%*$ & $71.97 \pm 2.73\%*$ & $55.93 \pm 2.34\%*$ & $67.47 \pm 1.11\%*$ \\
			LW-Linear & $\mathbf{76.50 \pm 4.16\%}$ & $\underline{46.34 \pm 2.72\%}*$ & $\underline{72.33 \pm 3.29\%}*$ & $\mathbf{57.29 \pm 2.37\%}$ & $\mathbf{68.67 \pm 1.05\%}$ \\
			\hline
			RC-MLP & $63.45 \pm 5.03\%*$ & $51.60 \pm 2.53\%*$ & $73.11 \pm 4.45\%*$ & $53.87 \pm 1.96\%*$ & $63.84 \pm 0.65\%*$ \\
			CC-MLP & $65.29 \pm 4.15\%*$ & $50.97 \pm 3.05\%*$ & $70.97 \pm 3.66\%*$ & $\underline{54.03 \pm 1.77\%}*$ & $62.85 \pm 1.26\%*$ \\
			PRODEN-MLP & $61.41 \pm 5.20\%*$ & $50.54 \pm 3.37\%*$ & $72.74 \pm 4.78\%*$ & $53.48 \pm 1.74\%*$ & $61.88 \pm 0.96\%*$ \\
			LW-MLP & $\underline{66.00 \pm 4.10\%}*$ & $\mathbf{52.23 \pm 3.61\%}$ & $\mathbf{73.89 \pm 4.01\%}$ & $53.64 \pm 1.83\%*$ & $\underline{64.65 \pm 0.98\%}*$ \\
			\bottomrule
		\end{tabular}
	}
	\begin{tablenotes}
		\footnotesize
		\item  The best results among all methods are marked in \textbf{bold} and the best under the same base model is marked in \underline{underline}.
		The standard deviation is also reported. 
		We use $*$ to represent that the best method is significantly better than the other compared methods.
	\end{tablenotes}
\end{table*}

\subsubsection{Benchmark dataset comparisons} \label{sec::simu}

\textbf{Datasets.}
We base our experiments on four benchmark datasets: MNIST \citep{lecun1998gradient}, Kuzushiji-MNIST \citep{clanuwat2018deep}, Fashion-MNIST \citep{xiao2017fashion}, and CIFAR-10 \citep{krizhevsky2009learning}.
We generate partially labeled data by making $K-1$ independent decisions for labels $z \neq y$, where each label $z$ has probability $q_z$ to enter the partial label set. 
In this part we consider $q_z=q$ for all $z \neq y$, where $q \in \{0.1, 0.3, 0.5\}$ and larger $q$ indicates that the partially labeled data is more ambiguous.
We put the experiments based on non-uniform data generating procedures in Section \ref{sec::datagenerate}.
Note that the true label $y$ always resides in the partial label set $\vec{y}$ and we accept the occasion that $\vec{y}=[K]$.
On MNIST, Kuzushiji-MNIST, and Fashion-MNIST, we employ the base model as a $5$-layer perception (MLP).
On the CIFAR-10 dataset, we employ a $12$-layer ConvNet \citep{laine2016temporal}
for all compared methods.
More details are shown in Section \ref{sec::Adata} of the supplements.

\textbf{Compared methods.}
We compare with the state-of-the-art PRODEN \citep{lv2020progressive}, RC and CC \citep{feng2020provably}, with all hyper-parameters searched according to the suggested parameter settings in the original papers. 
For our proposed method, we search the initial learning rate from $\{0.001, 0.005, 0.01, 0.05, 0.1\}$ and weight decay from $\{10^{-6}, 10^{-5}, \ldots, 10^{-2}\}$, with the exponential learning rate decay halved per $50$ epochs. 
We search $\beta \in \{1,2\}$ according to the theoretical guidance discussed in Section \ref{sec::thm_interpret}.
For computational implementations, we use {\tt PyTorch} \citep{paszke2019pytorch} and the stochastic gradient descent (SGD) \citep{robbins1951stochastic} optimizer with momentum $0.9$.
For all methods, we set the mini-batch size as $256$ and train each model for $250$ epochs.
Hyper-parameters are searched to maximize the accuracy on a validation set containing $10\%$ of the partially labeled training samples.
We adopt the same base model for fair comparisons.
More details are shown in Section \ref{sec::Amethod} of the supplements.

\textbf{Experimental results.}
We repeat all experiments $5$ times, and report the average accuracy and the standard deviation.
We apply the Wilcoxon signed-rank test  \citep{wilcoxon1992individual} at the significance level $\alpha=0.05$.
As is shown in Table \ref{tab::SynCompare}, 
when adopting the Sigmoid loss function with fine symmetric theoretical property, our proposed LW loss outperforms almost all other state-of-the-art algorithms for learning with partial labels. 
Moreover, by adopting the widely used cross entropy loss function, the empirical performance of LW can be further significantly improved on MNIST, Fashion-MNIST, and Kuzushiji-MNIST datasets.
We attribute this satisfactory result to the design of a proper leveraging parameter $\beta$, which makes it possible to consider the information of both partial labels and non-partial ones.

\subsubsection{Real Data Comparisons} \label{sec::real}

\textbf{Datasets.}
In this part we base our experimental comparisons on $5$ real-world datasets including: Lost \citep{cour2011learning}, MSRCv2 \citep{liu2012conditional}, BirdSong \citep{briggs2012rank}, Soccer Player \citep{zeng2013learning}, and Yahoo! News \citep{guillaumin2010multiple}.

\textbf{Compared methods.}
Aside from the network-based methods mentioned in Section \ref{sec::simu}, we compare with $3$ other state-of-the-art partial label learning algorithms including IPAL \citep{zhang2015solving},  PALOC \citep{wu2018towards}, and PLECOC \citep{zhang2017disambiguation}, where the hyper-parameters are searched through a $5$-fold cross-validation under the suggested settings in the original papers.
We adopt cross entropy loss for LW and employ both linear model and MLP as base models.
For all compared methods, we adopt a $10$-fold cross-validation to evaluate the testing performances.
Other settings are similar to Section \ref{sec::simu}.

\textbf{Experimental results.}
In Table \ref{tab::RealCompare}, under the same base model, our proposed LW shows the best performance on almost all datasets.
Moreover, on all real datasets, LW loss with proper base models always outperforms other state-of-the-art methods. 
Different from the benchmark datasets, the distribution of real partial labels remains unknown and could be more complex. 
Since our proposed LW loss is risk consistent with desired supervised loss functions under a generalized partial label generation assumption, there is no surprise that it presents satisfactory empirical performance.

\subsection{Empirical Understandings}\label{sec::emp_understand}

In this part, we conduct a series of comprehensive experiments to verify the effectiveness of our proposed LW loss. 

\begin{table*}[htbp!] 
	\setlength{\tabcolsep}{9pt}
	\centering
	\caption{Accuracy comparisons with different data generation.}
	\label{tab::caseCompare}
		\resizebox{\textwidth}{!}{
	\begin{tabular}{cl clll}
		\toprule
		Dataset & Method & Base Model& \qquad Case 1 & \qquad Case 2 & \qquad Case 3  \\
		\midrule
		\multirow{5}*{MNIST} & RC & MLP & $98.49 \pm 0.05\%*$ & $98.53 \pm 0.08\%*$ & $98.43 \pm 0.03\%*$ \\
		& CC & MLP & $98.55 \pm 0.04\%*$ & $98.57 \pm 0.08\%*$ & $98.44 \pm 0.02\%*$ \\
		& PRODEN & MLP & $98.64 \pm 0.15\%*$ & $97.61 \pm 0.10\%*$ & $98.55 \pm 0.12\%*$ \\
		& LW-Sigmoid & MLP & $\underline{98.83 \pm 0.04\%}$ & $\mathbf{98.92 \pm 0.04\%}$ & $\underline{98.69 \pm 0.11\%}$ \\
		& LW-Cross entropy & MLP & $\mathbf{98.88 \pm 0.05\%}$ & $\underline{98.88 \pm 0.09\%}$ & $\mathbf{98.82 \pm 0.05\%}$ \\
		\hline
		\multirow{5}*{Kuzushiji-MNIST} & RC & MLP & $92.61 \pm 0.17\%*$ & $92.47 \pm 0.19\%*$ & $92.07 \pm 0.10\%*$ \\
		& CC & MLP & $92.65 \pm 0.15\%*$ & $92.68 \pm 0.10\%*$ & $91.91 \pm 0.15\%*$ \\
		& PRODEN & MLP & $93.33 \pm 0.20\%*$ & $93.48 \pm 0.33\%*$ & $92.30 \pm 0.15\%*$ \\
		& LW-Sigmoid & MLP & $\underline{93.80 \pm 0.15\%}$ & $\underline{93.87 \pm 0.14\%}*$ & $\underline{93.09 \pm 0.19\%}*$ \\
		& LW-Cross entropy & MLP & $\mathbf{94.03 \pm 0.09\%}$ & $\mathbf{94.23 \pm 0.08\%}$ & $\mathbf{93.55 \pm 0.10\%}$ \\
		\hline
		\multirow{5}*{Fashion-MNIST} & RC & MLP & $89.79 \pm 0.10\%*$ & $89.88 \pm 0.11\%*$ & $89.47 \pm 0.11\%*$ \\
		& CC & MLP & $89.63 \pm 0.12\%*$ & $89.58 \pm 0.20\%*$ & $88.63 \pm 0.33\%*$ \\
		& PRODEN & MLP & $\underline{90.34 \pm 0.19\%}*$ & $89.88 \pm 0.27\%*$ & $89.60 \pm 0.14\%*$ \\
		& LW-Sigmoid & MLP & $90.24 \pm 0.04\%*$ & $\underline{90.32 \pm 0.18\%}$ & $\underline{89.69 \pm 0.21\%}*$ \\
		& LW-Cross entropy & MLP & $\mathbf{90.59 \pm 0.19\%}$ & $\mathbf{90.36 \pm 0.15\%}$ & $\mathbf{90.13 \pm 0.11\%}$ \\
		\hline
		\multirow{5}*{CIFAR-10} & RC & ConvNet & $86.59 \pm 0.34\%*$ & $87.26 \pm 0.06\%*$ & $86.28 \pm 0.17\%*$ \\
		& CC & ConvNet & $86.45 \pm 0.34\%*$ & $86.87 \pm 0.14\%*$ & $84.63 \pm 0.40\%*$ \\
		& PRODEN & ConvNet & $89.03 \pm 0.59\%*$ & $88.19 \pm 0.10\%*$ & $87.16 \pm 0.13\%*$ \\
		& LW-Sigmoid & ConvNet & $\mathbf{90.89 \pm 0.10\%}$ & $\mathbf{90.87 \pm 0.11\%}$ & $\underline{89.26 \pm 0.19\%}*$ \\
		& LW-Cross entropy & ConvNet & $\underline{90.63 \pm 0.08\%}*$ & $\underline{90.51 \pm 0.14\%}*$ & $\mathbf{89.60 \pm 0.09\%}$ \\
		\bottomrule 
	\end{tabular}
		}
	\begin{tablenotes}
		\footnotesize	
		\item {*} The best results are marked in \textbf{bold} and the second best marked in \underline{underline}. The standard deviation is also reported. 
		We use $*$ to represent that the best method is significantly better than the other compared methods.
	\end{tablenotes}
\end{table*}

\subsubsection{Parameter analysis}\label{sec::parameter}

We study the leverage parameter $\beta$ of LW loss by comparing its performances under $\beta \in \{0,1,2,4,8,16,32\}$ respectively.
We employ Sigmoid loss function for LW loss, and 
other experimental settings are similar to Section \ref{sec::simu}.

\begin{figure}[!h]
	\centering
	\subfigure[MNIST, $q=0.1$.]{
		\centering
		\includegraphics[width = 0.22\textwidth, trim= 0 200 0 240, clip]{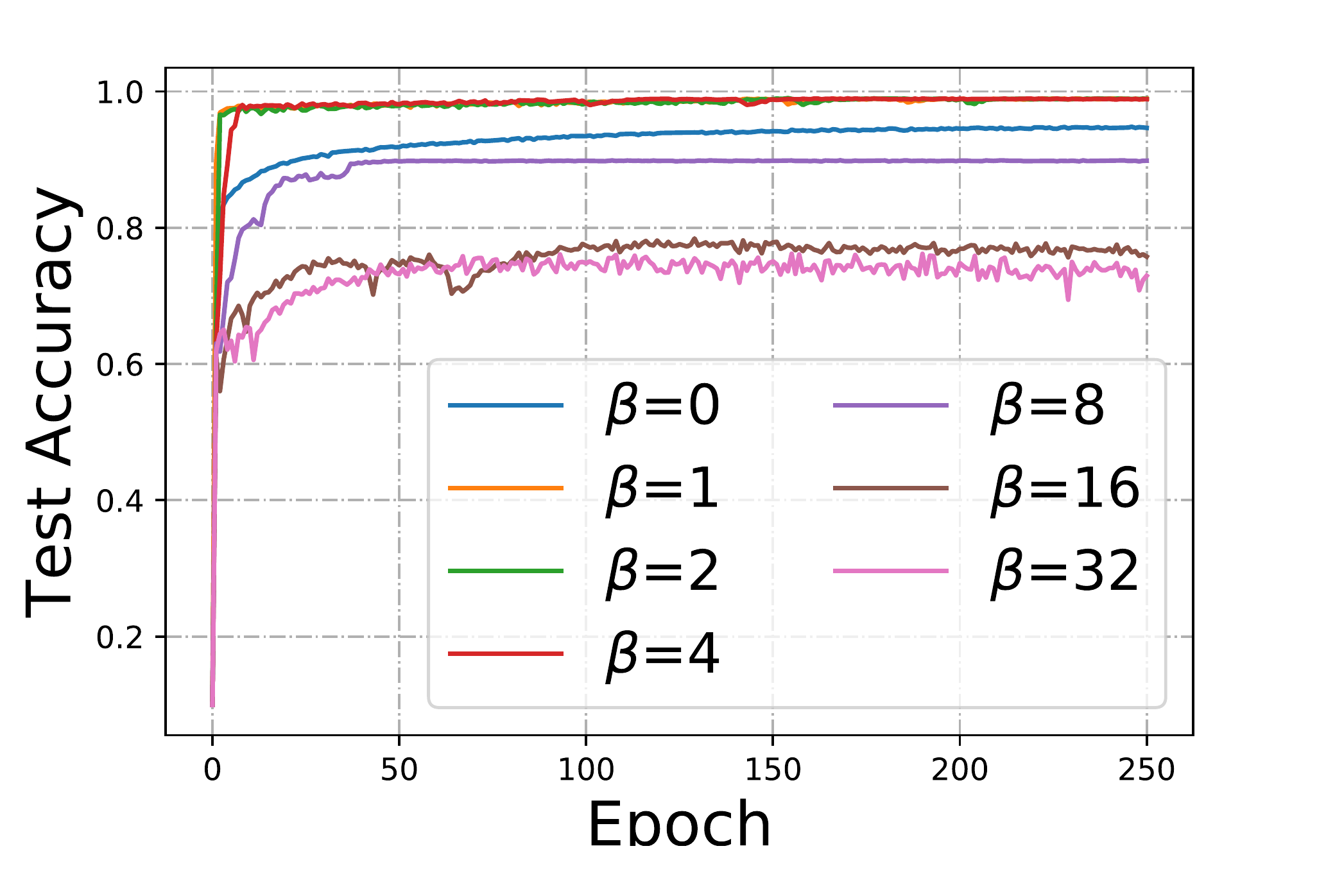}
	}
	\subfigure[Fashion-MNIST, $q=0.3$.]{
		\centering
		\includegraphics[width = 0.22\textwidth, trim= 0 200 0 240, clip]{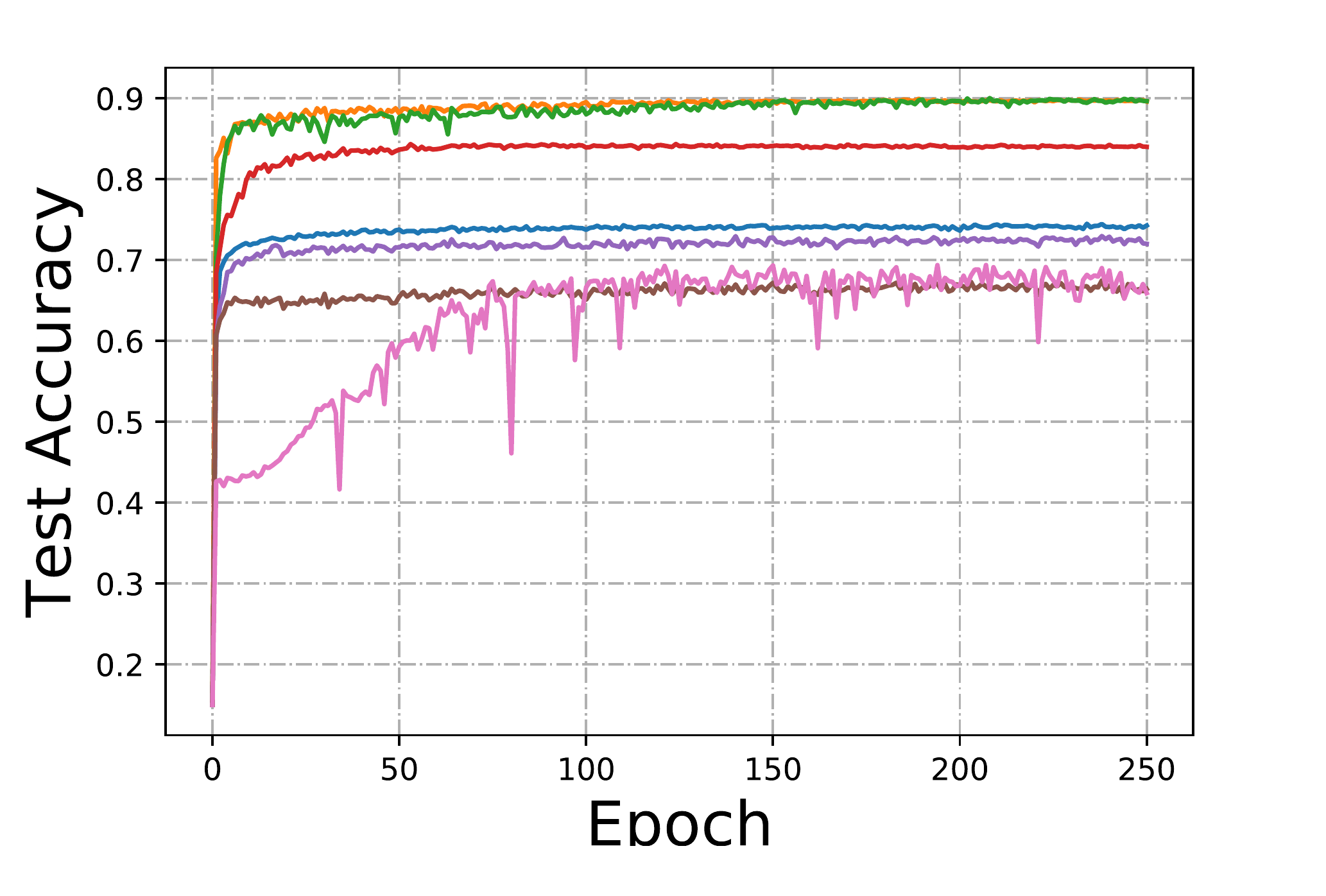}
	}
\\
	\subfigure[CIFAR-10, $q=0.3$.]{
		\centering
		\includegraphics[width = 0.22\textwidth, trim= 0 200 0 240, clip]{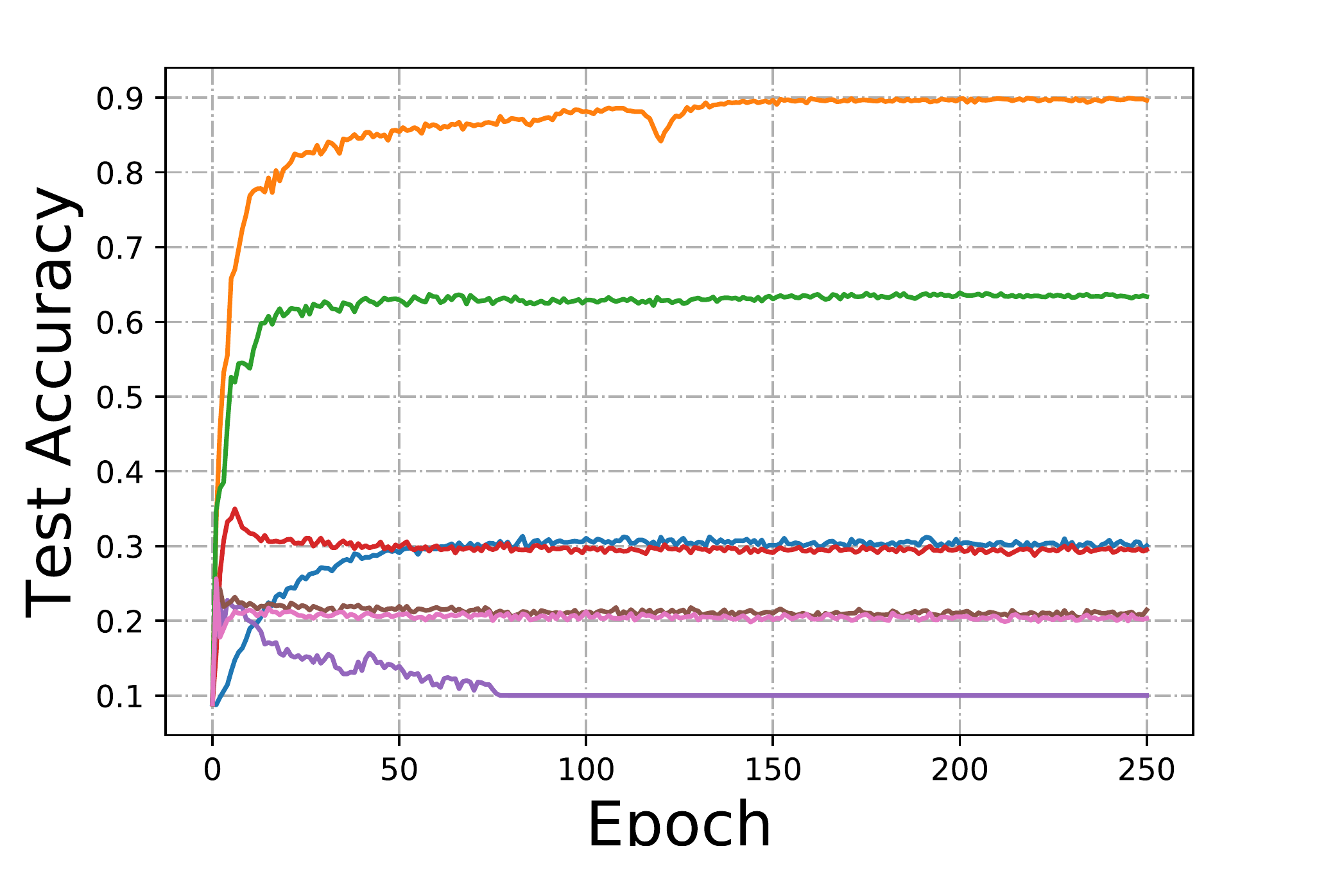}
	}
	\subfigure[Kuzushiji-MNIST, $q=0.5$.]{
		\centering
		\includegraphics[width = 0.22\textwidth, trim= 0 200 0 240, clip]{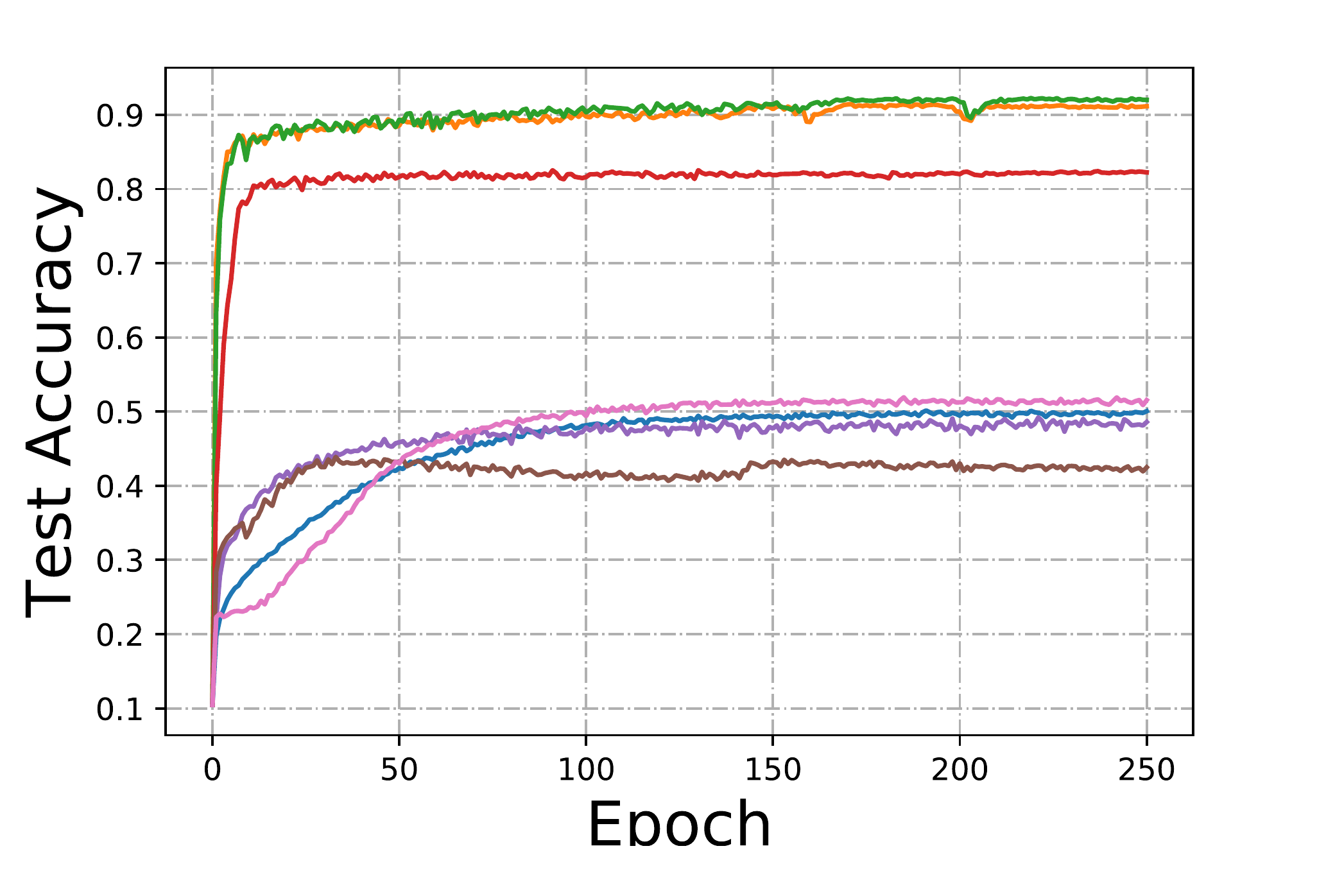}
	}
	\caption{Study of the leverage parameter $\beta$ for LW loss.}
	\label{fig::beta}
\end{figure}

As is shown in Figure \ref{fig::beta}, on all four datasets with varying data generation probability $q$, LW losses with $\beta = 1$ and $\beta=2$ significantly outperform those with other parameter settings. 
(On MNIST, LW loss with $\beta=4$ also performs competitively.)
This coincides exactly with the theoretical guidance to the choice of $\beta$ discussed in Section \ref{sec::thm_interpret}.

\subsubsection{Ablation Study}

In this part, we conduct an ablation study on effect of the two parts in our proposed LW loss, i.e. losses on partial labels $\sum_{z\in\vec{y}} w_z \psi(g_z(x))$ and those on non-partial ones $\sum_{z\notin\vec{y}} w_z \psi(-g_z(x))$. 
For notational simplicity, we rewrite the ``generalized'' LW loss as 
\begin{align*}
	\alpha \cdot \sum_{z\in\vec{y}} w_z \psi(g_z(x)) + \beta \cdot \sum_{z\notin\vec{y}} w_z \psi(-g_z(x)).
\end{align*} 
We compare among performances of LW with \\
1) losses on partial labels only ($\beta = 0$), \\
2) losses on non-partial labels only ($\alpha = 0$), \\
3) losses on both partial and non-partial labels ($\alpha \beta \neq 0$).
We employ the Sigmoid loss function for the LW loss.
Other experimental settings are similar to Section \ref{sec::simu}.

\begin{figure}[!h]
	\centering
	\subfigure[MNIST, $q=0.1$.]{
		\centering
		\includegraphics[width = 0.22\textwidth, trim= 0 200 0 240, clip]{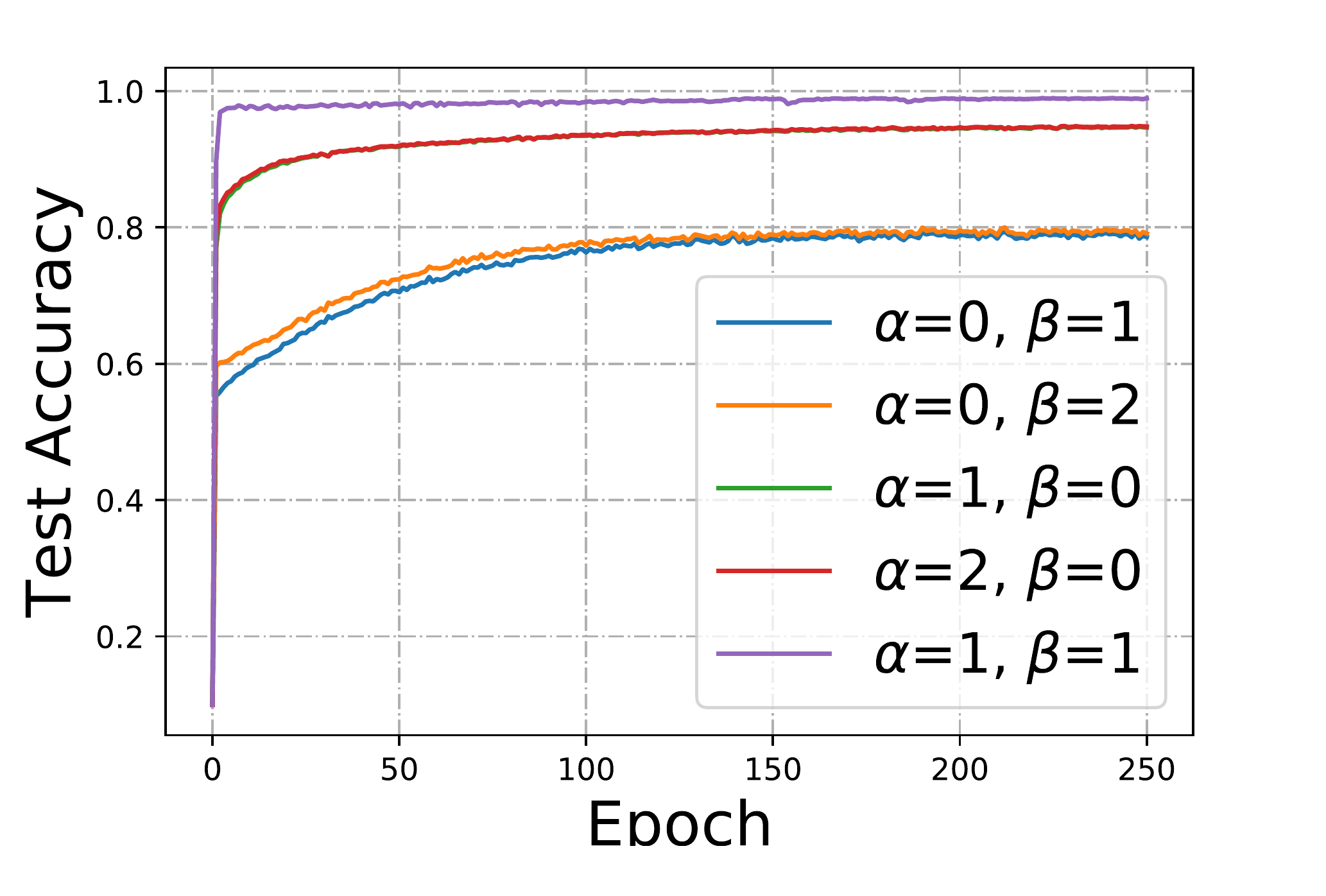}
	}
	\subfigure[Fashion-MNIST, $q=0.3$.]{
		\centering
		\includegraphics[width = 0.22\textwidth, trim= 0 200 0 240, clip]{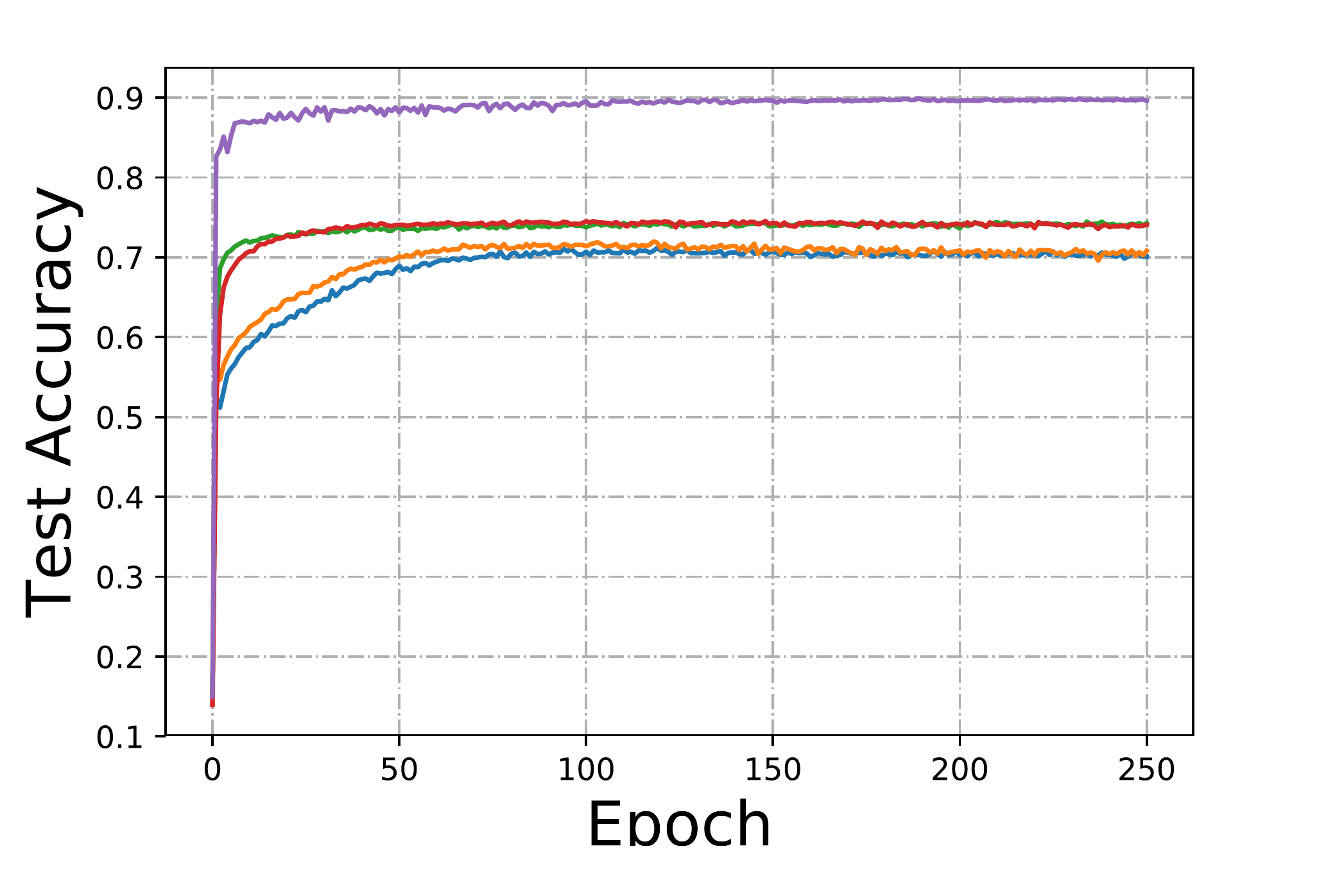}
	}
	\\
	\subfigure[CIFAR-10, $q=0.3$.]{
		\centering
		\includegraphics[width = 0.22\textwidth, trim= 0 200 0 240, clip]{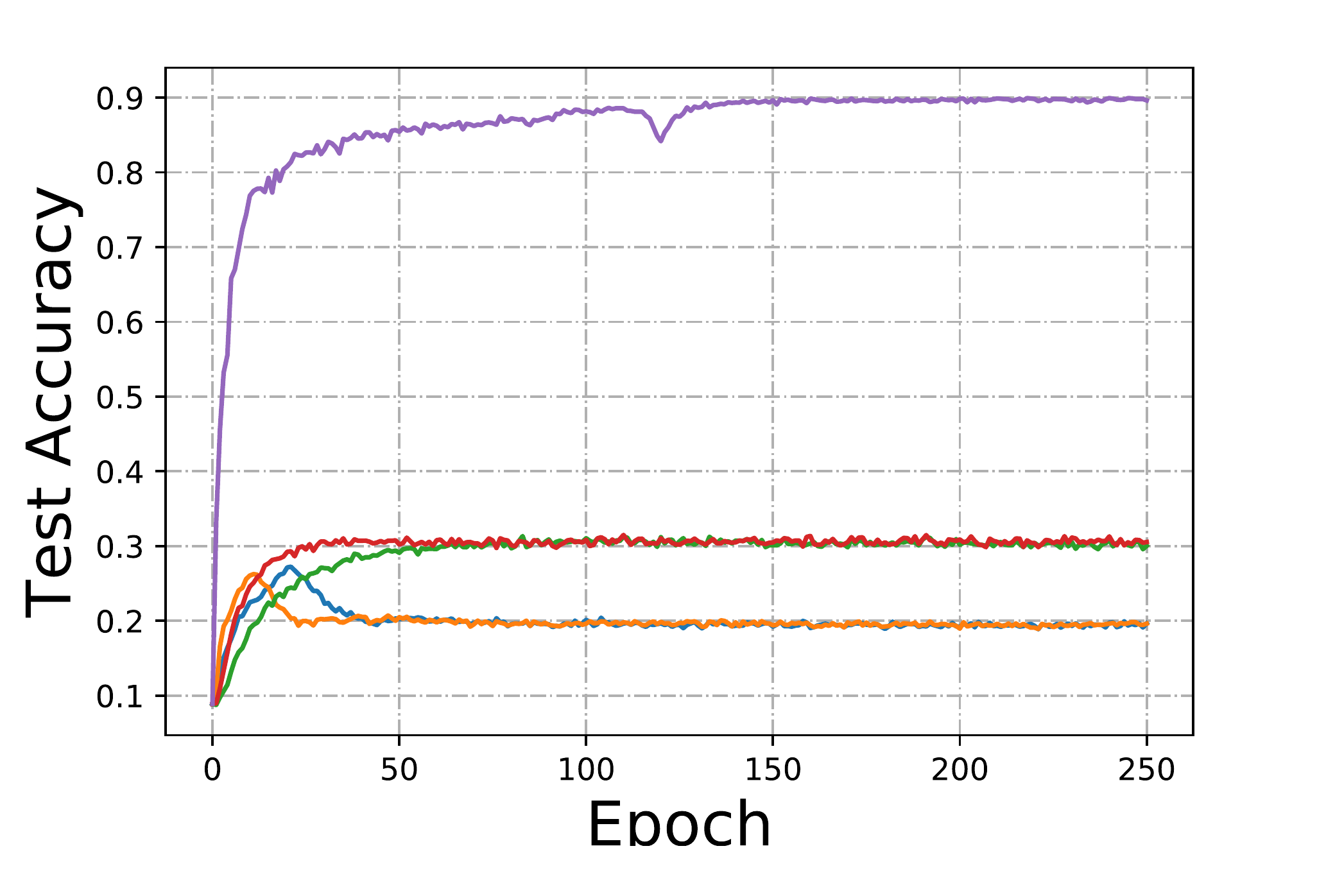}
	}
	\subfigure[Kuzushiji-MNIST, $q=0.5$.]{
		\centering
		\includegraphics[width = 0.22\textwidth, trim= 0 200 0 240, clip]{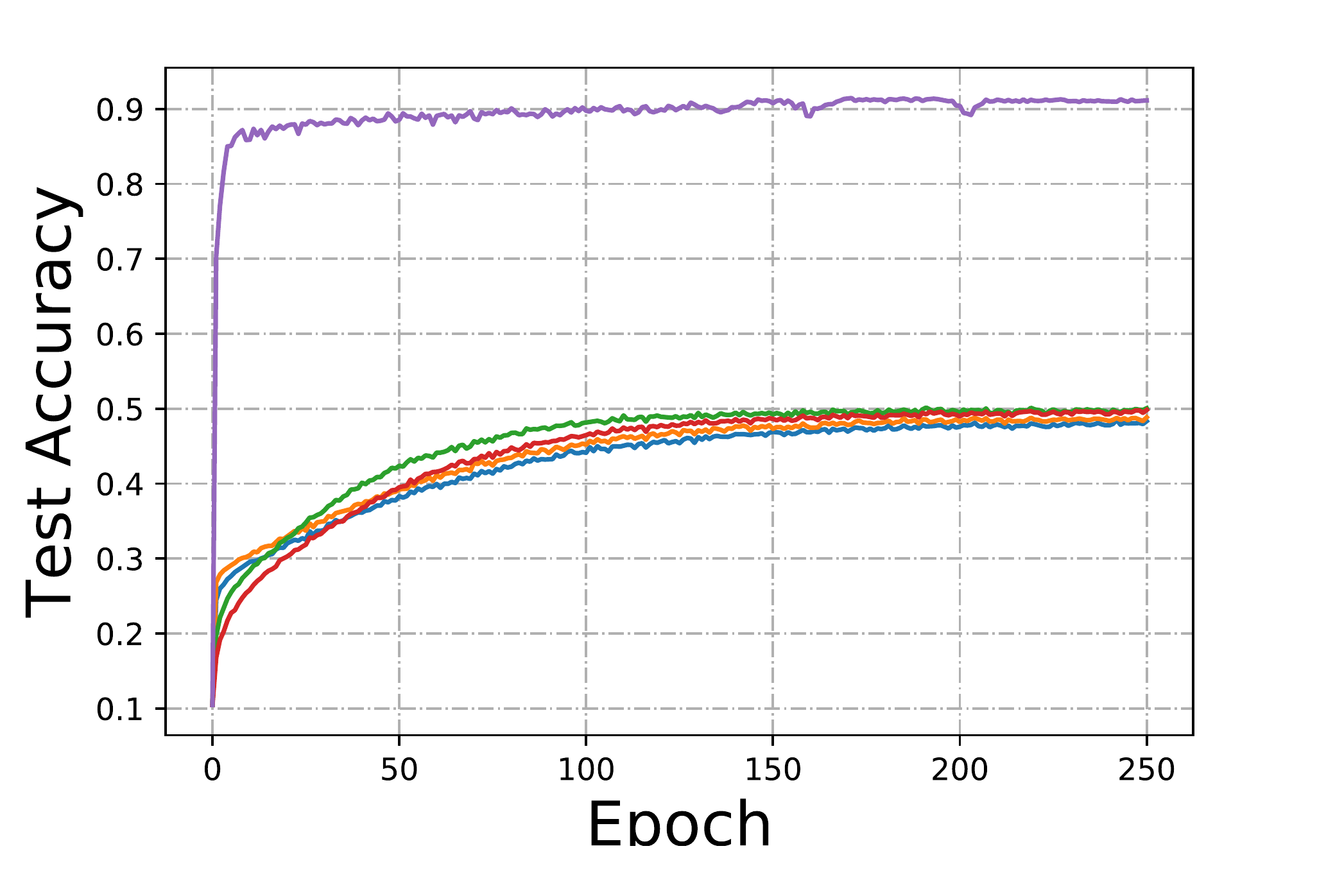}
	}
	\caption{Ablation study: comparisons between LW loss with losses on partial or non-partial labels.}
	\label{fig::ablation}
\end{figure}

As is shown in Figure \ref{fig::ablation}, when individually using losses on either partial labels or non-partial ones, the accuracy results is far from satisfactory on all three datasets since the information contained in the other half is neglected. 
Besides, it provides little help to the empirical performance by simply scaling the losses themselves.
On the contrary, by combining losses on both partial labels and non-partial ones ($\alpha = \beta = 1$), our proposed LW loss function shows its superiority in empirical performances,
where results show that our idea is especially effective on CIFAR-10 and Kuzushiji-MNIST datasets.

\subsubsection{The Influence of Data Generation}\label{sec::datagenerate}

In the data generation of previous subsections, the untrue partial labels are selected with equal probabilities, i.e. $q_z = q$ for $z \neq y$.
In reality, however, some labels may be more analogous to the true label than others, and thus the probabilities $q_z$ for these labels may naturally be higher than others.
In this part, we conduct empirical comparisons on data with alternative generation process.
To be specific, Case 1 describes a ``pairwise'' partial label set, where there exists only one potential partial label for each class. 
In Case 2, we assume two potential partial labels for each class.
Case 3 considers a more complex situation where $6$ potential labels have different probabilities to enter the partial label set.
More details about the data generations are shown in Section \ref{sec::Amatrix} in the supplementary material.
Other experimental settings are similar to Section \ref{sec::simu}.

As is shown in Table \ref{tab::caseCompare}, our proposed method dominates its counterparts in all three cases. 
Moreover, as the data generation process becomes more complex (from Case 1 to Case 3), there is a natural drop in accuracy for all methods.
Nonetheless, our LW-Cross entropy shows stronger resistance.
For example, on Kuzushiji-MNIST, in Case 1 the accuracy of our LW-Cross entropy is $0.71\%$ higher than PRODEN, while in Case 3 the difference increases to $1.25\%$.

\section{Conclusion}

In this paper, we propose a family of loss functions, named \textit{Leveraged Weighted} (LW) loss function, to address the problem of learning with partial labels. 
On the one hand, we provide theoretical guidance to the empirical choice of the leverage parameter $\beta$ proposed in our LW loss from the perspective of risk consistency.
Both theoretical interpretations and empirical understandings show that $\beta=1$ and $\beta=2$ are preferred parameter settings.
On the other hand, we design a practical algorithmic implementation of our LW loss, where its experimental comparisons with other state-of-the-art algorithms on both benchmark and real datasets demonstrate the effectiveness of our proposed method.

\section*{Acknowledgements}
Yisen Wang is supported by the National Natural Science Foundation of China under Grant No. 62006153, CCF-Baidu Open Fund (No. OF2020002), and Project 2020BD006 supported by PKU-Baidu Fund. Zhouchen Lin is supported by the National Natural Science Foundation of China (Grant No.s 61625301 and 61731018), Project 2020BD006 supported by PKU-Baidu Fund, Major Scientific Research Project of Zhejiang Lab (Grant No.s 2019KB0AC01 and 2019KB0AB02), and Beijing Academy of Artificial Intelligence.

\bibliographystyle{icml2021}
\small{\bibliography{partial}}

\clearpage
\section*{Appendix}

\appendix

This file consists of supplementaries for both theoretical analysis and experiments. 
In Section \ref{sec::Aproofs}, we present the proof of Theorem \ref{thm::weight} in Section \ref{sec::methodology}.
In Section \ref{sec::Aexp}, we present more detailed settings of the numerical experiments including descriptions of datasets, compared methods, model architecture, and data generation procedures. 

\section{Proofs}\label{sec::Aproofs}
We present all proofs for Section \ref{sec::methodology} here.
For the sake of conciseness and readability, we denote 
$\vec{\mathcal{Y}}^y$ as the collection of all partial label sets containing the true label $y$, i.e. $\vec{\mathcal{Y}}^y := \{\vec{y} \in \vec{\mathcal{Y}}| y \in \vec{y}\}$.


In order to achieve the risk consistency result for the LW loss in Theorem \ref{thm::weight}, we first present in Theorem \ref{thm::partialrisk} the risk consistency result for an arbitrary loss function $\bar{\mathcal{L}}(\vec{y}, g(x))$ under the generalized assumption that partial label sets follows the \textit{label-specific} sampling.

\begin{theorem}\label{thm::partialrisk}
	Denote $q_z := \mathrm{P}(z \in \vec{y} \,|\, Y=y,x)$. Then the partial loss function $\bar{\mathcal{L}}(\vec{y}, g(x))$ is risk-consistent with respect to the supervised loss function with the form
	\begin{align}
		\mathcal{L}(y, g(x)) = \sum_{\vec{y} \in \vec{\mathcal{Y}}^y} \prod_{s \in \vec{y}, s \neq y} q_s \prod_{t \notin \vec{y}} (1-q_t) \bar{\mathcal{L}}(\vec{y}, g(x)),
	\end{align}
	where $\vec{\mathcal{Y}}^y := \{\vec{y} \in \vec{\mathcal{Y}} \,|\, y \in \vec{y}\}$ denotes the partial label set containing label $y$.
\end{theorem}

\begin{proof}[of Theorem \ref{thm::partialrisk}]
	For any $x \in \mathcal{X}$, there holds
	\begin{align*}
		&\ \phantom{=}\bar{\mathcal{R}}(\bar{\mathcal{L}}, g(X)) 
		\\
		&= \mathbb{E}_{\vec{Y}|X}[\bar{\mathcal{L}}(\vec{Y}, g(x))|X=x]
		\\
		&= \sum_{\vec{y} \in 2^{[K]}} \bar{\mathcal{L}}(\vec{y}, g(x)) \mathrm{P}(\vec{Y}=\vec{y}|X=x)
		\\
		&= \sum_{\vec{y} \in 2^{[K]}} \bar{\mathcal{L}}(\vec{y}, g(x)) 
		\sum_{y\in\textbf{y}} \mathrm{P}(\vec{Y}=\vec{y}, Y= y|X=x)
		\\
		&= \sum_{\vec{y} \in 2^{[K]}} \bar{\mathcal{L}}(\vec{y}, g(x)) 
		\\
		&\ \phantom{=}\cdot \sum_{y\in\textbf{y}} \mathrm{P}(\vec{Y}=\vec{y} | Y= y, X=x) \mathrm{P}(Y=y|X=x)
		\\
		&= \sum_{y=1}^K \mathrm{P}(Y=y|X=x) 
		\\
		&\cdot \sum_{\vec{y} \in 2^{[K]}} \mathrm{P}(\vec{Y}=\vec{y}| Y= y, X=x) \bar{\mathcal{L}}(\vec{y}, g(x)),
	\end{align*}
	and
	\begin{align*}
		\mathcal{R}(\mathcal{L}, g(X)) 
		&= \mathbb{E}_{Y|X}[\mathcal{L}(Y, g(x))|X=x] 
		\\
		&= \sum_{y=1}^K \mathcal{L}(y, g(x))\mathrm{P}(Y=y|X=x).
	\end{align*}
	Since $\mathrm{P}(\vec{Y}=\vec{y}| Y= y, X=x) = 0$ for $\vec{y}$ not containing $y$, if we have
	\begin{align*}
		\mathcal{L}(y, g(x)) 
		&= \sum_{\vec{y} \in 2^{[K]}} \mathrm{P}(\vec{Y}=\vec{y}| Y= y, X=x) \bar{\mathcal{L}}(\vec{y}, g(x))
		\\	
		&= \sum_{\vec{y} \in \vec{\mathcal{Y}}^y} \mathrm{P}(\vec{Y}=\vec{y}| Y= y, X=x) \bar{\mathcal{L}}(\vec{y}, g(x)),
		\\
		&=\sum_{\vec{y} \in \vec{\mathcal{Y}}^y} \prod_{s \in \vec{y}, s \neq y} q_s \prod_{t \notin \vec{y}} (1-q_t) \bar{\mathcal{L}}(\vec{y}, g(x)),
	\end{align*}
	then there holds
	\begin{align*}
		\bar{\mathcal{R}}(\bar{\mathcal{L}}, g(X)) = \mathcal{R}(\mathcal{L}, g(X)).
	\end{align*}
\end{proof}

Besides, to prove Theorem \ref{thm::weight}, we need the following result shown in Lemma \ref{lem::prob_sum}.

\begin{lemma}\label{lem::prob_sum}
	Let $y$ be the true label of input $x$, $q_z := \mathrm{P}(z \in \vec{y} | Y=y, X=x)$ for $z \in \mathcal{Y}$, and $\vec{\mathcal{Y}}^y := \{\vec{y} \in \vec{\mathcal{Y}} | y \in \vec{y}\}$.
	Then there holds
	\begin{align*}
		\sum_{\vec{y} \in \vec{\mathcal{Y}}^y} \prod_{s \in \vec{y}, s \neq y} q_s \prod_{t \notin \vec{y}} (1-q_t) = 1.
	\end{align*}
\end{lemma}

\begin{proof}[of Lemma \ref{lem::prob_sum}]
	Since $q_y=1$, we have
	\begin{align*}
		&\ \phantom{=}\sum_{\vec{y} \in \vec{\mathcal{Y}}^y} \prod_{s \in \vec{y}, s \neq y} q_s \prod_{t \notin \vec{y}} (1-q_t)
		\\
		&= \sum_{\vec{y} \in \vec{\mathcal{Y}}^y} \prod_{s \in \vec{y}, s \neq y} 1 \cdot q_s \prod_{t \notin \vec{y}} (1-q_t)
		\\
		&= \sum_{\vec{y} \in \vec{\mathcal{Y}}^y} \prod_{s \in \vec{y}, s \neq y} q_y \cdot q_s \prod_{t \notin \vec{y}} (1-q_t)
		\\
		&= \sum_{\vec{y} \in \vec{\mathcal{Y}}^y} \prod_{s \in \vec{y}} q_s \prod_{t \notin \vec{y}} (1-q_t)
		\\
		&= \sum_{\vec{y} \in \vec{\mathcal{Y}}^y} \mathrm{P}(\vec{Y} = \vec{y} | Y=y, X=x)
		\\
		&= \sum_{\vec{y} \in \vec{\mathcal{Y}}} \mathrm{P}(\vec{Y} = \vec{y} | Y=y, X=x)
		\\
		&= 1,
	\end{align*}
	where the second last equation holds since $\mathrm{P}(\vec{Y} = \vec{y} | Y=y, X=x) = 0$ for $\vec{y} \notin \vec{\mathcal{Y}}^y$.
	
\end{proof}

\begin{proof}[of Theorem \ref{thm::weight}]
	According to Theorem \ref{thm::partialrisk}, we have the partial loss function $\bar{\mathcal{L}}_{\psi}$ consistent with
	\begin{align}\label{eq::weight}
		&\phantom {=} \mathcal{L}_{\psi}(y,g(x)) 
		\nonumber\\
		&= \sum_{\vec{y} \in \vec{\mathcal{Y}}^y} \prod_{s \in \vec{y}, s \neq y} q_s \prod_{t \notin \vec{y}} (1-q_t) \bar{\mathcal{L}}_{\psi}(\vec{y}, g(x))
		\nonumber\\
		&= \sum_{\vec{y} \in \vec{\mathcal{Y}}^y} \prod_{s \in \vec{y}, s \neq y} q_s \prod_{t \notin \vec{y}} (1-q_t) \sum_{z \in \vec{y}} w_z \psi(g_z(x)) 
		\nonumber\\
		&+ \beta \cdot \sum_{\vec{y} \in \vec{\mathcal{Y}}^y} \prod_{s \in \vec{y}, s \neq y} q_s \prod_{t \notin \vec{y}} (1-q_t) \sum_{z \notin \vec{y}} w_z \psi(-g_z(x)).
	\end{align}
	
	The first term on the right hand side of \eqref{eq::weight} is
	\begin{align}\label{eq::first_term}
		&\phantom{=} \, \sum_{\vec{y} \in \vec{\mathcal{Y}}^y} \prod_{s \in \vec{y}, s \neq y} q_s \prod_{t \notin \vec{y}} (1-q_t) \sum_{z \in \vec{y}} w_z \psi(g_z(x))
		\nonumber\\
		&= \sum_{\vec{y} \in \vec{\mathcal{Y}}^y} \prod_{s \in \vec{y}, s \neq y} q_s \prod_{t \notin \vec{y}} (1-q_t) w_y \psi(g_y(x)) 
		\nonumber\\
		&+ \sum_{\vec{y} \in \vec{\mathcal{Y}}^y} \prod_{s \in \vec{y}, s \neq y} q_s \prod_{t \notin \vec{y}} (1-q_t) \sum_{z \in \vec{y}\setminus \{y\}} w_z \psi(g_z(x))
		\nonumber\\
		&= \sum_{\vec{y} \in \vec{\mathcal{Y}}^y} \prod_{s \in \vec{y}, s \neq y} q_s \prod_{t \notin \vec{y}} (1-q_t) w_y \psi(g_y(x)) 
		\nonumber\\
		&+ \sum_{\vec{y} \in \vec{\mathcal{Y}}^y} \sum_{z \in \vec{y}\setminus \{y\}} \prod_{s \in \vec{y}, s \neq y} q_s \prod_{t \in [K]\setminus \vec{y}} (1-q_t)  w_z \psi(g_z(x)).
	\end{align}
	
	By Lemma \ref{lem::prob_sum}, we have
	\begin{align*}
		\sum_{\vec{y} \in \vec{\mathcal{Y}}^y} \prod_{s \in \vec{y}, s \neq y} q_s \prod_{t \notin \vec{y}} (1-q_t) = 1,
	\end{align*}
	and therefore the first term in \eqref{eq::first_term} becomes
	\begin{align}\label{eq::first_first}
		\sum_{\vec{y} \in \vec{\mathcal{Y}}^y} \prod_{s \in \vec{y}, s \neq y} q_s \prod_{t \notin \vec{y}} (1-q_t) w_y \psi(g_y(x)) 
		&= w_y \psi(g_y(x)).
	\end{align}
	
	For the second term in \eqref{eq::first_term},
	since $z \neq y$ and $z \in \vec{y}$, we switch the summations, and achieve
	\begin{align*}
		&\phantom{=} \, \sum_{\vec{y} \in 2^{[K]}} \sum_{z \in \vec{y}\setminus \{y\}} \prod_{s \in \vec{y}, s \neq y} q_s \prod_{t \in [K] \setminus \vec{y}} (1-q_t)  w_z \psi(g_z(x))
		\\
		&= \sum_{z \neq y} \sum_{\vec{y} \in \vec{\mathcal{Y}}^z \cap \vec{\mathcal{Y}}^y} \prod_{s \in \vec{y}, s \neq y} q_s \prod_{t \in [K] \setminus \vec{y}} (1-q_t)  w_z \psi(g_z(x))
		\\
		&= \sum_{z \neq y} w_z \psi(g_z(x)) \sum_{\vec{y} \in \vec{\mathcal{Y}}^z \setminus \vec{\mathcal{Y}}^y} \prod_{s \in \vec{y}} q_s \prod_{t \in [K] \setminus \vec{y} \setminus \{y\}} (1-q_t).
	\end{align*}
	
	Without loss of generality, we assume $y = K$ for notational simplicity, and write
	\begin{align*}
		&\phantom{=} \, \sum_{z \neq y} w_z \psi(g_z(x)) \sum_{\vec{y} \in \vec{\mathcal{Y}}^z \setminus \vec{\mathcal{Y}}^y} \prod_{s \in \vec{y}} q_s \prod_{t \in [K] \setminus \vec{y} \setminus \{y\}} (1-q_t)
		\\
		&= \sum_{z \in [K-1]} w_z \psi(g_z(x)) 
		\sum_{\vec{y} \in (2^{[K-1]})^z} \prod_{s \in \vec{y}} q_s \prod_{t \in [K-1] \setminus \vec{y}} (1-q_t)
		\\
		&= \sum_{z \in [K-1]} w_z \psi(g_z(x)) q_z 
		\\
		&\cdot \sum_{\vec{y} \in (2^{[K-1]})^z} \prod_{s \in \vec{y}, s\neq z} q_s \prod_{t \in [K-1] \setminus \vec{y}} (1-q_t).
	\end{align*}
	
	Applying Lemma \ref{lem::prob_sum} with $\vec{\mathcal{Y}} = 2^{[K-1]}$, we have 
	\begin{align}\label{eq::sum_prob_k-1}
		\sum_{\vec{y} \in (2^{[K-1]})^z} \prod_{s \in \vec{y}, s\neq z} q_s \prod_{t \in [K-1] \setminus \vec{y}} (1-q_t) = 1,
	\end{align}
	and therefore the second term in \eqref{eq::first_term} becomes
	\begin{align}\label{eq::first_second}
		&\phantom{=}\sum_{\vec{y} \in 2^{[K]}} \sum_{z \in \vec{y}\setminus \{y\}} \prod_{s \in \vec{y}, s \neq y} q_s \prod_{t \in [K] \setminus \vec{y}} (1-q_t)  w_z \psi(g_z(x))
		\nonumber\\
		&= \sum_{z \neq y} q_z w_z \psi(g_z(x)).
	\end{align}
	
	Similarly, by switching the summations, the second term on the right hand side of \eqref{eq::weight} becomes
	\begin{align}\label{eq::second}
		&\phantom{=} \, \beta \cdot \sum_{\vec{y} \in \vec{\mathcal{Y}}^y} \prod_{s \in \vec{y}, s \neq y} q_s \prod_{t \notin \vec{y}} (1-q_t) \sum_{z \notin \vec{y}} w_z \psi(-g_z(x))
		\nonumber\\
		&= \beta \cdot \sum_{\vec{y} \in \vec{\mathcal{Y}}^y} \sum_{z \notin \vec{y}} \prod_{s \in \vec{y}, s \neq y} q_s \prod_{t \notin \vec{y}} (1-q_t) w_z \psi(-g_z(x))
		\nonumber\\
		&= \beta \cdot \sum_{z \neq y} \sum_{\vec{y} \in \vec{\mathcal{Y}}^y \cap \vec{\mathcal{Y}}^z}  \prod_{s \in \vec{y}, s \neq y} q_s \prod_{t \notin \vec{y}} (1-q_t) w_z \psi(-g_z(x))
		\nonumber\\
		&= \beta \cdot \sum_{z \neq y} \sum_{\vec{y} \in \vec{\mathcal{Y}}^z \setminus \vec{\mathcal{Y}}^y}  \prod_{s \in \vec{y}} q_s \prod_{t \notin \vec{y}, t\neq y} (1-q_t) w_z \psi(-g_z(x))
		\nonumber\\
		&= \beta \cdot \sum_{z \neq y} \sum_{\vec{y} \in (2^{[K-1]})^z} \prod_{s \in \vec{y}} q_s \prod_{t \notin \vec{y}} (1-q_t) w_z \psi(-g_z(x))
		\nonumber\\
		&= \beta \cdot \sum_{z \neq y} q_z \sum_{\vec{y} \in (2^{[K-1]})^z} \prod_{s \in \vec{y}, s \neq z} q_s \prod_{t \notin \vec{y}} (1-q_t) w_z \psi(-g_z(x))
		\nonumber\\
		&= \beta \cdot \sum_{z \neq y} q_z w_z \psi(-g_z(x)),
	\end{align}
	where the last equality holds according to \eqref{eq::sum_prob_k-1}.

	By combining \eqref{eq::first_first}, \eqref{eq::first_second}, \eqref{eq::second}, we have
	\begin{align*}
		\mathcal{L}_{\psi}(y,g(x)) &= w_y \psi(g_y(x)) 
		+ \sum_{z \neq y} q_z w_z \psi(g_z(x)) 
		\\
		&+ \beta \cdot \sum_{z \neq y} q_z w_z \psi(-g_z(x))
		\\
		&= w_y \psi(g_y(x)) 
		\\
		&+ \sum_{z \neq y} q_z w_z \big[\psi(g_z(x)) + \beta\psi(-g_z(x))\big].
	\end{align*}
	%
	
\end{proof}

Before proving Theorem \ref{thm::calibration}, we define the inner risk of loss function $\mathcal{L}_{\psi}$ by
\begin{align*}
	\mathcal{C}_{\mathcal{L}_\psi}(g) := \mathbb{E}_{Y|X} \mathcal{L}_{\psi}(Y, g(X)) = \sum_{y \in [K]} p_y \mathcal{L}_{\psi}(Y, g(X)),
\end{align*}
where $p_y := \mathrm{P}(Y=y \,| x)$.

\begin{proof}[of Theorem \ref{thm::calibration}]
	By Theorem \ref{thm::weight}, we can write the inner risk induced by the supervised loss $\mathcal{L}_{\psi}$ as 
	\begin{align*}
		\mathcal{C}_{\mathcal{L}_\psi}(g) &:=\mathbb{E}_{Y|X} \mathcal{L}_{\psi}(Y, g(X))
		\\
		&= \sum_{y \in [K]} p_y \Big(w_y q_y \psi(g_y(x)) 
		\\
		&\quad  + \sum_{z \neq y} w_z q_z \big[\psi(g_z(x)) + \beta \psi(-g_z(x))\big]\Big)
		\\
		&= \sum_{y \in [K]} p_y w_y q_y \psi(g_y(x)) 
		\\
		&\quad  + \sum_{y \in [K]} p_y \sum_{z \neq y} w_z q_z \big[\psi(g_z(x)) + \beta \psi(-g_z(x))\big]
		\\
		&= \sum_{y \in [K]} p_y w_y q_y \psi(g_y(x)) 
		\\
		&\quad  + \sum_{z \in [K]} \sum_{y \neq z} p_y w_z q_z \big[\psi(g_z(x)) + \beta \psi(-g_z(x))\big]
		\\
		&= \sum_{y \in [K]} p_y w_y q_y \psi(g_y(x)) 
		\\
		&\quad  + \sum_{y \in [K]} \sum_{z \neq y} p_z w_y q_y \big[\psi(g_y(x)) + \beta \psi(-g_y(x))\big]
		\\
		&= \sum_{y\in [K]} \Big( p_y w_y q_y \psi(g_y(x)) 
		\\
		&\quad + (1-p_y) w_y q_y \big[\psi(g_y(x)) + \beta \psi(-g_y(x))\big] \Big),
		\\
	\end{align*}
	where $p_y := \mathrm{P}(Y=y\,|\,x)$.
	
	Due to the symmetric property of $\psi(\cdot)$, we have 
	\begin{align*}
		\mathcal{C}_{\mathcal{L}_{\psi}}(g) 
		&= \sum_{y\in [K]} \Big( p_y w_y q_y \psi(g_y(x)) 
		\\
		&\quad + (1-p_y) w_y q_y \big[\beta + (1-\beta)\psi(g_y(x))\big] \Big)
		\\
		&= \sum_{y \in [K]} w_yq_y(\beta p_y-(\beta-1))\psi(g_y)+C_1,
	\end{align*}
	where $C_1:=\sum_{y \in [K]} \beta(1-p_y)w_y q_y$.
	
	Next, we consider the constraint comparison method (CCM) \cite{lee2004multicategory} defined by
	\begin{align*}
		\mathcal{L}_{CCM}(y,g(x)) := \sum_{k \neq y} \psi(-g_k(x))
	\end{align*}
	with the constraint $\sum_{k\in[K]} g_k = 0$.
	The inner risk induced by the constraint comparison method (CCM) has the form 
	\begin{align*}
		\mathcal{C}_{CCM}(g) &:= \mathbb{E}_{Y|X} \mathcal{L}_{CCM}(X, Y)
		\\
		&= \sum_{y\in[K]}  p_y \sum_{z \neq y} \psi(-g_z)
		\\
		&= \sum_{y\in[K]} \sum_{z \neq y} p_y \psi(-g_z)
		\\
		&= \sum_{z\in[K]} \sum_{y \neq z} p_y \psi(-g_z)
		\\
		&= \sum_{y\in[K]} \sum_{z \neq y} p_z \psi(-g_y)
		\\
		&= \sum_{y\in[K]} (1-p_y) \psi(-g_y).
	\end{align*}
	Since $\psi(\cdot)$ is symmetric, we have
	\begin{align*}
		\mathcal{C}_{CCM}(g)
		&= \sum_{y\in[K]}  (1-p_y) (1-\psi(g_y))
		\\
		&= \sum_{y\in[K]} (p_y-1)\psi(g_y)+C_2,
	\end{align*}
	where $C_2 := 1-p_y$. 
	
	Denote $y^* := \max_{y\in[K]} p_y$. We have $p_{y^*} = 1$.
	By Section \ref{sec::main_alg}, we have $\argmax_{y\in[K]} w_y = \argmax_{y\in[K]} p_y$. 
	Then when $\beta >0$, there holds 
	\begin{align*}
		\argmax_{y\in[K]}w_yq_y(\beta p_y-(\beta-1)) = \argmax_{y\in[K]}(p_y-1).
	\end{align*} 
	which implies optimizing $\mathcal{L}_{LW}$ and $\mathcal{L}_{CCM}$ achieves the same classifier.
	According to Example 3 in Section 5.3 of \cite{tewari2007consistency}, when $\psi$ is differentiable, ${\mathcal{L}}_{CCM}$ is proved to be consistent in the multi-class classification setting. Therefore, optimizing \eqref{eq::lv2020progressive_loss} will also lead to the Bayes classifier, which implies when there holds
	\begin{align*}
		\mathcal{R}(\mathcal{L}_{\psi},\hat{g}_n) \to \mathcal{R}^*_{\mathcal{L}_{\psi}}
	\end{align*}
	there also holds
	\begin{align}\label{eq::bayes}
		\mathcal{R}(\mathcal{L}_{0\text{-}1},\hat{g}_n) \to \mathcal{R}^*,
	\end{align}
	where $\mathcal{L}_{0\text{-}1}$ is the multi-class supervised loss. 
	This finishes the proof. 
\end{proof}

%
%

\section{Supplementary for Experiments}  \label{sec::Aexp}

\subsection{Descriptions of Datasets} \label{sec::Adata}
\subsubsection{Benchmark Datasets}\label{sec::Abenchmark}
In Section \ref{sec::simu}, we use four widely-used benchmark datasets, i.e. MNIST\citep{lecun1998gradient}, Kuzushiji-MNIST \citep{clanuwat2018deep}, Fashion-MNIST\citep{xiao2017fashion}, CIFAR-10\citep{krizhevsky2009learning}. 
The characteristics of these datasets are reported in Table \ref{tab::FourBenchmark}. 
We concisely describe these nine datasets as follows.
\begin{itemize}
	\item MNIST: It is a 10-class dataset of handwritten digits, i.e. 0 to 9. Each data is a $28\times28$ grayscale image.
	\item Fashion-MNIST: It is also a 10-class dataset. Each instance is a fashion item from one of the 10 classes, which are T-shirt/top, trouser, pullover, dress, sandal, coat, shirt, sneaker, bag, and ankle boot. Moreover, each image is a $28\times28$ grayscale image.
	\item Kuzushiji-MNIST: Each instance is a $28\times28$ grayscale image associated with one label of 10-class cursive Japanese (“Kuzushiji”) characters. 
	\item CIFAR-10: Each instance is a $32\times32\times3$ colored image in RGB format. It is a ten-class dataset of objects including airplane, bird, automobile, cat, deer, frog, dog, horse, ship, and truck.
\end{itemize}

\begin{table*}[htbp] 
	\setlength{\tabcolsep}{10pt}
	\centering
	\caption{Summary of benchmark datasets.}
	\label{tab::FourBenchmark}
	\begin{tabular}{p{3cm}p{2.0cm}ccccc}
		\toprule
		Dataset & $\#$ Train & $\#$ Test & $\#$ Feature & $\#$ Class  
		\\
		\midrule
		MNIST & $60,000$ & $10,000$ & $784$ & $10$
		\\
		\hline
		Kuzushiji-MNIST   & $60,000$ & $10,000$ & $784$ & $10$
		\\
		\hline
		Fashion-MNIST  & $60,000$ & $10,000$ & $784$ & $10$
		\\
		\hline
		CIFAR-10 & $50,000$ & $10,000$ & $3,072$ & $10$
		\\
		\bottomrule 
	\end{tabular}
\end{table*}

\subsubsection{Real Datasets} \label{sec::Areal}
In Section \ref{sec::real}, we use five real-world partially labeled datasets (Lost, BirdSong, MSRCv2, Soccer Player, Yahoo!$\ $News). 
Detailed descriptions are shown as follows.

\begin{itemize}
	\item Lost, Soccer Player and Yahoo!$\ $News: They corp faces in images or video frames as instances, and the names appearing on the corresponding captions or subtitles are considered as candidate labels.
	\item MSRCv2: Each image segment is treated as a sample, and objects appearing in the same image are regarded as candidate labels.
	\item BirdSong: Birds' singing syllables are regarded as instances and bird species who are jointly singing during any ten seconds are represented as candidate labels.
\end{itemize}

Tabel \ref{tab::RealData}  includes the average number of candidate labels (Avg. $\#$ CLs) per instance. 

\begin{table*}[htbp] 
	\setlength{\tabcolsep}{9pt}
	\centering
	\caption{Summary of real-world partial label datasets.}
	\label{tab::RealData}
	\begin{tabular}{lccccccc}
		\toprule
		Dataset & $\#$ Examples & $\#$ Features & $\#$ Class & Avg $\#$ CLs  & Task Domain
		\\
		\midrule
		Lost & $1,122$ & $108$ & $16$ & $2.23$  &  Automatic face naming \\
		\hline
		BirdSong  & $4,998$ & $38$ & $13$ & $2.18$ & Bird song classification  
		\\
		\hline
		MSRCv2 & $1,758$ & $48$ & $23$ & $3.16$ & Object classification 
		\\
		\hline
		Soccer Player & $17,472$ & $279$  & $171$  & $2.09$ & Automatic face naming
		\\
		\hline
		Yahoo! News & $22,991$ & $163$ & $219$ & $1.91$ & Automatic face naming 
		\\
		\bottomrule 
	\end{tabular}
\end{table*}

\subsection{Compared Methods}\label{sec::Amethod}
The compared partial label methods are listed as follows.

IPAL \citep{zhang2015solving} : It is a non-parametric method that uses the label propagation strategy to iteratively update the confidence of each candidate label. The suggested configuration is as follows: the balancing coefficient $\alpha = 0.95$, the number of nearest neighbors considered $k = 10$, and the number of iterations $T = 100$. 

PALOC (\citep{wu2018towards}): It adapts the popular one-vs-one decomposition strategy to solve the partial label problem. The suggested configuration is the balancing coefficient $\mu = 10$ and the SVM model. 

PLECOC (\citep{zhang2017disambiguation}): It transforms the partial label learning problem to a binary label problem by E-COC coding matrix. The suggested configuration is codeword length $L = \lceil10 \log2(q)\rceil$ and SVM model. Moreover, the eligibility parameter $\tau$ is set to be one-tenth of the number of training instances (i.e. $\tau = |D|/10$). 

Hyper-parameters for these three methods are selected
through a 5-fold cross-validation.

Next, we list three compared partial label methods based on neural network models. 

PRODEN(\citep{lv2020progressive}): It propose a novel estimator of the classification risk and  a progressive identification algorithm for approximately minimizing the proposed risk estimator. 
The parameters is selected through grid search, where the learning rate $lr \in \{10^{-5}, 10^{-4}, \ldots, 10^{-1}\}$ and weight decay $wd \in \{10^{-6}, 10^{-4}, \ldots, 10^{-2}\}$. The optimizer is stochastic gradient descent (SGD) with momentum 0.9. 

RC $\&$ CC(\citep{feng2020provably}): The former method is a novel risk-consistent partial label learning method and the latter one is classifier-consistent based on the generation model. 
For the two methods, the suggested parameter grids of learning rate and weight decay are both $\{10^{-6}, 10^{-5}, \ldots, 10^{-1}\}$. They are implemented by PyTorch and the Adam optimizer. 

For all these three compared methods, hyper-parameters are selected so as to maximize the accuracy on a validation set, constructed by randomly sampling $10\%$ of the training set. The mini-batch size is set as $256$ and the number of epochs is set as $250$. They all apply the cross-entropy loss function to build the partial label loss function.

\subsection{Details of Architecture}
In this section, we list the architecture of three models, linear, MLP, and ConvNet. 
The linear model is a linear-in-input model: $d-10$.
MLP refers to a $5$-layer fully connected networks with ReLU as the activation function, whose architecture is  $d-300-300-300-300-10$. Batch normalization was applied before hidden layers. For both models, the softmax function was applied to the output layer, and $\ell_2$-regularization was added. 

The detailed architecture of ConvNet \citep{laine2016temporal} is as follows.

0th (input) layer: (32*32*3)-

1st to 4th layers: [C(3*3, 128)]*3-Max Pooling-

5th to 8th layers: [C(3*3, 256)]*3-Max Pooling-

9th to 11th layers: C(3*3, 512)-C(3*3, 256)-C(3*3, 128)-

12th layers: Average Pooling-10

where C(3*3, 128) means 128 channels of 3*3 convolutions followed by Leaky-ReLU (LReLU) active function, $[\cdot]*3$ means 3 such layers, etc.

%

\subsection{Matrix Representations of Alternative Data Generations}\label{sec::Amatrix}

\textbf{Case 1:}
Each true label has a unique similar label with probability $q_1>0$ to enter the partial label set, while all other labels are not partial labels. When $q_1 = 0.5$, the data generation corresponds to the one proposed in \cite{lv2020progressive}.
A matrix representation is 
\begin{align*}
	\begin{bmatrix}
		1 & q_1 & 0 & 0 & 0 & 0 & 0 & 0 & 0 & 0 \\
		0 & 1 & q_1 & 0 & 0 & 0 & 0 & 0 & 0 & 0 \\
		0 & 0 & 1 & q_1 & 0 & 0 & 0 & 0 & 0 & 0 \\
		0 & 0 & 0 & 1 & q_1 & 0 & 0 & 0 & 0 & 0 \\
		0 & 0 & 0 & 0 & 1 & q_1 & 0 & 0 & 0 & 0 \\
		0 & 0 & 0 & 0 & 0 & 1 & q_1 & 0 & 0 & 0 \\
		0 & 0 & 0 & 0 & 0 & 0 & 1 & q_1 & 0 & 0 \\
		0 & 0 & 0 & 0 & 0 & 0 & 0 & 1 & q_1 & 0 \\
		0 & 0 & 0 & 0 & 0 & 0 & 0 & 0 & 1 & q_1 \\
		q_1 & 0 & 0 & 0 & 0 & 0 & 0 & 0 & 0 & 1 \\
	\end{bmatrix}
\end{align*}
where the element in the $i$-th row and the $j$-th column represents the conditional probability $\mathrm{P}(j \in \vec{Y} \,|\, Y=i, x)$.

\textbf{Case 2:} 
Each true label has two similar labels with probability $q_1>0$ to be partial labels, while all other labels are not partial labels. Here we let $q_1=0.3$.
A matrix representation is 
\begin{align*}
	\begin{bmatrix}
		1 & q_1 & 0 & 0 & 0 & 0 & 0 & 0 & 0 & q_1 \\
		q_1 & 1 & q_1 & 0 & 0 & 0 & 0 & 0 & 0 & 0 \\
		0 & q_1 & 1 & q_1 & 0 & 0 & 0 & 0 & 0 & 0 \\
		0 & 0 & q_1 & 1 & q_1 & 0 & 0 & 0 & 0 & 0 \\
		0 & 0 & 0 & q_1 & 1 & q_1 & 0 & 0 & 0 & 0 \\
		0 & 0 & 0 & 0 & q_1 & 1 & q_1 & 0 & 0 & 0 \\
		0 & 0 & 0 & 0 & 0 & q_1 & 1 & q_1 & 0 & 0 \\
		0 & 0 & 0 & 0 & 0 & 0 & q_1 & 1 & q_1 & 0 \\
		0 & 0 & 0 & 0 & 0 & 0 & 0 & q_1 & 1 & q_1 \\
		q_1 & 0 & 0 & 0 & 0 & 0 & 0 & 0 & q_1 & 1 \\
	\end{bmatrix}
\end{align*}



\textbf{Case 3:} 
In this case, we allow more pairs of similar labels.
For each true label, there exist a pair of most similar labels with probability $q_1$ to be partial labels, two pairs of less similar labels with probabilities $q_2$ and $q_3$ respectively. 
Assume that $q_1 > q_2 > q_3 > 0$.
Other labels are taken as non-partial labels. 
We let $q_1 = 0.5$, $q_2 = 0.3$, $q_3 = 0.1$. A matrix representation is 
\begin{align*}
	\begin{bmatrix}
		1 & q_1 & q_2 & q_3 & 0 & 0 & 0 & q_3 & q_2 & q_1 \\
		q_1 & 1 & q_1 & q_2 & q_3 & 0 & 0 & 0 & q_3 & q_2 \\
		q_2 & q_1 & 1 & q_1 & q_2 & q_3 & 0 & 0 & 0 & q_3 \\
		q_3 & q_2 & q_1 & 1 & q_1 & q_2 & q_3 & 0 & 0 & 0 \\
		0 & q_3 & q_2 & q_1 & 1 & q_1 & q_2 & q_3 & 0 & 0 \\
		0 & 0 & q_3 & q_2 & q_1 & 1 & q_1 & q_2 & q_3 & 0 \\
		0 & 0 & 0 & q_3 & q_2 & q_1 & 1 & q_1 & q_2 & q_3 \\
		q_3 & 0 & 0 & 0 & q_3 & q_2 & q_1 & 1 & q_1 & q_2 \\
		q_2 & q_3 & 0 & 0 & 0 & q_3 & q_2 & q_1 & 1 & q_1 \\
		q_1 & q_2 & q_3 & 0 & 0 & 0 & q_3 & q_2 & q_1 & 1 \\
	\end{bmatrix}
\end{align*}


\end{document}